\documentclass[12pt]{arxiv} 
\usepackage{natbib} 
\usepackage{mathtools} 
\usepackage{booktabs} 
\usepackage{tikz} 
\usepackage[american]{babel}
\usepackage[draft]{tikzpeople}
\usepackage{amssymb}
\usepackage{cleveref}

\usetikzlibrary{shapes.geometric,shapes.symbols, automata, matrix}

\newcommand{\myset}[1]{\{ #1 \}}
\newcommand{\tuple}[1]{\langle #1 \rangle}
\newcommand{\setNelems}[1]{\{ #1_i \}_{i=1}^N}
\newcommand{\setNTelems}[1]{\{ #1_i^{(t)} \}_{i=1,t=1}^{N,T}}

\newcommand{\parents}{\mathbf{\mathrm{Pa}}}

\newcommand{\desc}{\mathbf{\mathrm{Desc}}}

\newcommand{\scmu}[1]{\mathcal{M}_{#1}}
\newcommand{\scmdagu}[1]{\mathcal{G}_{\scm{}_{#1}}}
\newcommand{\envarsu}[1]{\mathcal{X}_{#1}}
\newcommand{\exvarsu}[1]{\mathcal{U}_{#1}}
\newcommand{\structfuncsu}[1]{\mathcal{F}_{#1}}
\newcommand{\probdistsu}[1]{\mathcal{P}_{#1}}
\newcommand{\scmsignatureu}[1]{\tuple{ \envarsu{#1}, \exvarsu{#1}, \structfuncsu{#1}, \probdistsu{#1} }}

\newcommand{\scm}{\scmu{}}
\newcommand{\scmdag}{\scmdagu{}}
\newcommand{\envars}{\envarsu{}}
\newcommand{\exvars}{\exvarsu{}}
\newcommand{\structfuncs}{\structfuncsu{}}
\newcommand{\probdists}{\probdistsu{}}
\newcommand{\scmsignature}{\scmsignatureu{}}

\newcommand{\doopt}{\textrm{do}}
\newcommand{\dointv}[2]{\doopt(#1=#2)}

\newcommand{\ctfop}{\textrm{ctf}}
\newcommand{\ctf}[4]{\ctfop(\cdot_{#3=#4}\vert#1=#2)}
\newcommand{\ctfinv}[2]{\ctfop(\cdot_{#1=#2})}
\newcommand{\ctfscm}{\scmu{\ctfop}}
\newcommand{\ctfscmsignature}{\scmsignatureu{\ctfop}}

\newcommand{\iiop}{\textrm{do}^\star}
\newcommand{\iiintv}[2]{\iiop(#1=#2)}
\newcommand{\finscm}{\scmu{\textrm{fin}}}
\newcommand{\finscmsignature}{\scmsignatureu{\textrm{fin}}}

\newcommand{\abd}{\textrm{abd}}
\newcommand{\twin}{\textrm{twin}}
\newcommand{\rep}{\textrm{rep}}
\newcommand{\fin}{\textrm{fin}}

\newcommand{\dd}[1]{``#1''}

\newtheorem{assumption}{Assumption}
\newtheorem{problem}{Problem Statement}
\newtheorem{mydef}{Definition}
\crefname{mydef}{def}{defs}
\Crefname{mydef}{Definition}{Definiions}
\newtheorem{myprop}{Proposition}
\newtheorem{mylemma}{Lemma}

\newcommand{\rebuttal}[1]{{#1}}

\newcommand{\narrative}[1]{}

\title{Teleological Inference in Structural Causal Models via Intentional Interventions}

\clearauthor{\Name{Dario Compagno} \Email{dario.compagno@parisnanterre.fr}\AND
 \Name{Fabio Massimo Zennaro} \Email{fabio.zennaro@uib.no}\\
 \addr }
  
\begin{document}
  
\maketitle

\begin{abstract}
Structural causal models (SCMs) were conceived to formulate and answer causal questions. This paper shows that SCMs can also be used to formulate and answer teleological questions, concerning the intentions of a state-aware, goal-directed agent intervening in a causal system. We review limitations of previous approaches to modeling such agents, and then introduce intentional interventions, a new time-agnostic operator that induces a twin SCM we call a structural final model \rebuttal{(SFM)}. SFMs treat observed values as the outcome of intentional interventions and relate them to the counterfactual conditions of those interventions (what would have happened had the agent not intervened). We show how SFMs can be used to empirically detect agents and to discover their intentions.
\end{abstract}

\section{Introduction}\label{sec:intro}

\narrative{Interventions and causal modeling.}
Contemporary causal research looks for causal relationships by realizing or simulating \textit{interventions} into causal systems. Indeed, \dd{no causation without manipulation} is a well-known motto, implying a certain notion of causality and causal research \citep{holland1986}.
Several mathematical formalisms complying with this understanding have been suggested \citep{rubin1974estimating,pearl2009causality,dawid2021decision}; among them, the framework of \emph{structural causal models} (SCM) \citep{pearl2009causality, peters2017elements} is widely adopted in artificial intelligence (AI) and machine learning (ML). Actual or simulated actions taken by experimenters are means for putting causal hypotheses to the test.

Can agents and their interventions themselves also become the object of modeling? The study of agents and their actions is a legitimate object of scientific research (especially in biology, psychology and the social sciences) although not necessarily well captured as causal systems \citep{anscombe1957, vonwright1971, dennett1987}. 
\narrative{Limitation of causal modeling in teleological analysis.}
AI and ML researchers are interested too in studying \emph{intentions} \citep{ward2024reasons, friedenberg2025intents}, that is, explaining how a state-aware and goal-driven artificial agent affects and is affected by its interactions with a causal system. In particular, researchers may be interested in knowing when spurious dependencies are introduced by agents (\emph{agent detection}) or when specific outcomes in the system explain the behavior of an agent (\emph{intention discovery}). However, standard SCMs do not model agents nor intentions, and they need to be extended to account for intentional interventions. Consider the following examples:

\begin{example}[Has somebody turned on the heater?]\label{ex:heating_causal}
We are interested in modeling the heating system in a house through an SCM. We consider three binary variables of interest: weather ($W$), the state of the heater ($H$), and room temperature ($T$). We express the relations of cause and effect as in \Cref{fig:basic-heating-system}(a). 
Suppose also that we collect data as in the table of \Cref{fig:basic-heating-system}(a). Is the \rebuttal{collected} data due to an agent intervening upon the heater?
\end{example}

\begin{example}[Why do people smoke?]\label{ex:smoking_causal}
We are interested in modeling smoking behavior ($S$) in association to two variables: pleasure ($P$) and lung damage \rebuttal{in the form of tar deposits} ($D$). We express the relations of cause and effect as in \Cref{fig:basic-heating-system}(b). 
Suppose also that there is an agent actually smoking and experiencing the pleasure and \rebuttal{tar deposits} caused by it. Does the agent smoke in order to get lung damage?
\end{example}

\begin{figure}
    \centering
    \begin{tikzpicture}   
        [node distance = 1cm, thick,
        arrow/.style = {draw, ->},
        circ/.style = {draw, ellipse}]
    
        \begin{scope}[xshift=0]       
            \node[circ, text=black, draw=black, fill=white] at (-1.2, 0.5) (W) {$W$};
            \node[circ, text=black, draw=black, fill=white] at (0, 0) (T) {$T$};
            \node[circ, text=black, draw=black, fill=white] at (1.2, 0.5) (H) {$H$};
    
            \node[alice,minimum size=.8cm] at (3,0.25) (Ag) {};
            
            \draw[arrow] (W) edge [bend left=0] (T);
            \draw[arrow] (H) edge [bend left=0] (T);
            \draw[arrow,dotted] (2.5,0.4) to (1.8,0.4);
            \draw[arrow, dotted] (1.8,0.2) to (2.5,0.2);
    
            \matrix (m) [matrix of nodes,
                   nodes in empty cells,
                   nodes={anchor=center,minimum width=.6cm,font=\small}] 
              at (5, 0.3) 
              {
                $W$ & $T$ & $H$ \\
                0 & 1 & 1 \\
                0 & 1 & 1 \\
                1 & 1 & 0 \\
                0 & 1 & 1 \\
              };
                \draw (m-1-1.north west) -- (m-5-1.south west);
                \draw (m-1-1.north east) -- (m-5-1.south east);
                \draw (m-1-2.north east) -- (m-5-2.south east);
                \draw (m-1-3.north east) -- (m-5-3.south east);
                \draw (m-2-1.north west) -- (m-2-3.north east);

            \node at (2,-1.5) {(a)};
        \end{scope}

        \begin{scope}[xshift=260]
            \node[circ, text=black, draw=black, fill=white] at (-1.2, 0) (P) {$P$};
            \node[circ, text=black, draw=black, fill=white] at (0, .5) (S) {$S$};
            \node[circ, text=black, draw=black, fill=white] at (1.2, 0) (D) {$D$};
    
            \node[alice,minimum size=.8cm] at (3,0.25) (Ag) {};
            
            \draw[arrow] (S) edge [bend left=0] (P);
            \draw[arrow] (S) edge [bend left=0] (D);
            \draw[arrow,dotted] (2.5,0.1) to (1.8,0.1);
            \draw[arrow, dotted] (1.8,0.3) to (2.5,0.3);

           \node at (1,-1.5) {(b)};
        \end{scope}
                
        \end{tikzpicture}
    \caption{(a) Causal model for a heating system defined over variables $W,T,H$ (for Weather, room Temperature, Heater status); we \rebuttal{collect} the data in the adjacent table. (b) Causal model for smoking defined over variables $S,P,D$ (for Smoking, Pleasure, lung Damage).}
    \label{fig:basic-heating-system}
\end{figure}
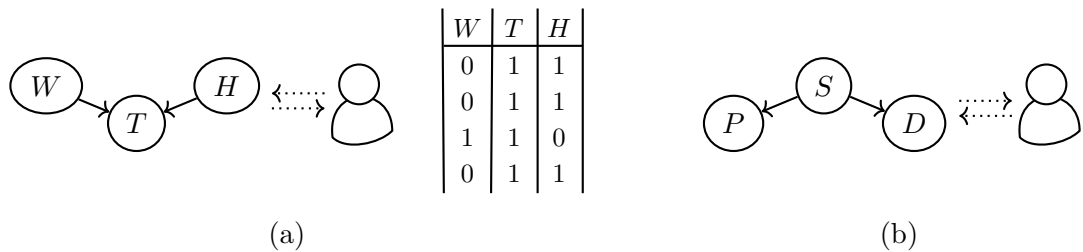

\narrative{Previous approaches and novelty of our approach.} The modeling adopted in \Cref{ex:heating_causal} and \Cref{ex:smoking_causal}, although grounded from a causal point of view, is unsuited to reason about agents.
In \Cref{ex:heating_causal}, \rebuttal{the dataset might show a} comfortable temperature under any weather condition, giving rise to a perfect dependency between weather and heating, but the agentless SCM fails to account for this dependency; that is, we cannot explain the data with the given SCM. 
In \Cref{ex:smoking_causal} we observe that pleasure and damage co-occur causally, but an SCM offers no way to express the dependence of an action on a \textit{subset} of its effects \citep{compagno2025}; that is, we cannot state whether or not an agent smokes in order to get pleasure and not lung damage. 

In this paper, we argue that existing modeling approaches based on SCMs deal with the problem of representing actions and agents in an unsatisfactory way, by factoring agents' intentions in exogenous or endogenous variables, or by switching to time-based models.
We suggest that these approaches have inherent conceptual shortcomings, stemming from the treatment of intentions as causal variables \citep{anscombe1957}.
We propose instead a general approach to the modeling of intentions in time-agnostic settings based on the definition of \emph{intentional interventions}, an operator akin to standard intervention, which can be applied to an SCM and which generates a new model in the form of a \textit{structural final model} \rebuttal{(SFM)}.

An intentional intervention captures the way a state-aware and goal-driven agent modifies a causal system, not just assigning an arbitrary value to a target variable (like standard interventions), but \dd{listening to} \citep{pearl2009causality} a set of variables in order to determine what value to assign. 
A SFM defined by an intentional intervention is constructed by the twinning of two SCMs: one describing what has happened as the result of an intentional intervention, and a second one describing what would have happened had the agent \textit{not} intervened.
Relative to \Cref{ex:heating_causal}, agents turn on heaters depending on the room temperature, meaning that they would not intervene were it already warm; relative to \Cref{ex:smoking_causal}, agents smoke depending on pleasure, meaning that they would not smoke if they did not derive pleasure (while they would keep smoking if this did not produce \rebuttal{tar deposits}).

Thus, from a teleological perspective, what determines the value of an intervened variable are not its \textit{factual ancestors} but its \textit{counterfactual descendants}. Intervened variables do not have causes in an SCM \citep{pearl2009causality}, but have intended effects which can be expressed as counterfactual conditions of the model.
We show through simple examples that SFMs can be used for \emph{teleological} \textit{inference}, especially in its two most basic forms: the empirical detection of agents (such as in \Cref{ex:heating_causal}) and the discovery of their intentions (such as in \Cref{ex:smoking_causal}). 
SFMs are thus designed as a framework to easily formulate teleological questions and answer them empirically with data and assumptions about their independencies.
Notice that the modeling strategy we propose is an \textit{extension} to the standard framework of SCMs and not a replacement; the standard way of building SCMs and perform causal inference remains crucial for performing teleological inferences.

\narrative{Potential impact of our work.}
Detecting that a causal system has been intentionally intervened by an agent (as in \Cref{ex:heating_causal}) is important in contexts such as computer security (is an agent probing available services?) or in digital domains (is a large language model interacting with tools?). Discovering the subset of effects motivating an action (as in \Cref{ex:smoking_causal}) is relevant in contexts including health (which effects of a treatment are so negatively perceived so to compromise its follow-ups?) or social policies (what makes citizens prefer cars over public transportation?). SFMs provide a new tool for explainability, opening the way to uncovering the aims of a system (\textit{teleological interpretability}) instead of the calculations leading towards that outcome (mechanistic interpretability). Moreover, as human behavior is often explained in terms of intentions \citep{bratman1987, tomasello1999}, the ability to model them formally and to validate teleological hypotheses empirically has evident implications.

\section{Background}
In this section we review the basic concepts about causality and SCMs. Further discussion on the notation, graph-theoretical definitions are available in \Cref{app:notation}.

\begin{figure}
    \centering
    \begin{tikzpicture}
        [node distance = .8cm, thick,
        arrow/.style = {draw, ->},
        circ/.style = {draw, ellipse},
        font=\tiny]

        \begin{scope}[xshift=0]
            \node[circ, text=black, draw=black, fill=white] at (0, 1.2) (X) {$X$};
            \node[circ, text=black, draw=black, fill=white] at (0, 0) (Y) {$Y$};
            \node[circ, text=black, draw=black, fill=white] at (0, 2.4) (Z) {$Z$};
            \node[circ, text=black, draw=black, fill=white, dashed] at (0.8, 1.8) (UX) {$U_X$};
            \node[circ, text=black, draw=black, fill=white, dashed] at (0.8, 0.6) (UY) {$U_Y$};
            \node[circ, text=black, draw=black, fill=white, dashed] at (0.8, 3.0) (UZ) {$U_Z$};
    
            \draw[arrow] (UX) edge [bend left=0] (X);
            \draw[arrow] (UY) edge [bend left=0] (Y);
            \draw[arrow] (UZ) edge [bend left=0] (Z);
            \draw[arrow] (X) edge [bend left=0] (Y);
            \draw[arrow] (Z) edge [bend left=0] (X);

            \node at (0.5,-0.9) {\normalsize (a)};
        \end{scope}

        \begin{scope}[xshift=60]
            \node[circ, text=black, draw=black, fill=white] at (0, 1.2) (X) {$1$};
            \node[circ, text=black, draw=black, fill=white] at (0, 0) (Y) {$Y$};
            \node[circ, text=black, draw=black, fill=white] at (0, 2.4) (Z) {$Z$};
            \node[circ, text=black, draw=black, fill=white, dashed] at (0.8, 1.8) (UX) {$U_X$};
            \node[circ, text=black, draw=black, fill=white, dashed] at (0.8, 0.6) (UY) {$U_Y$};
            \node[circ, text=black, draw=black, fill=white, dashed] at (0.8, 3.0) (UZ) {$U_Z$};
    
            \draw[arrow] (UY) edge [bend left=0] (Y);
            \draw[arrow] (UZ) edge [bend left=0] (Z);
            \draw[arrow] (X) edge [bend left=0] (Y);

            \node at (0.5,-0.9) {\normalsize (b)};
        \end{scope}

        \begin{scope}[xshift=120]
            \node[circ, text=black, draw=black, fill=white] at (0, 1.2) (X) {$1$};
            \node[circ, text=black, draw=black, fill=white] at (0, 0) (Y) {$Y$};
            \node[circ, text=black, draw=black, fill=white] at (0, 2.4) (Z) {$Z$};
            \node[circ, text=black, draw=black, fill=white, dashed] at (0.8, 1.8) (UX) {$U_X$};
            \node[circ, text=black, draw=black, fill=white, dashed] at (0.8, 0.6) (UY) {$U_Y$};
            \node[circ, text=black, draw=black, fill=white, dashed] at (0.8, 3.0) (UZ) {$U_Z$};
            \node[circ, text=gray, draw=gray, fill=white] at (1.6, 1.2) (X1) {$0$};
            \node[circ, text=gray, draw=gray, fill=white] at (1.6, 0) (Y1) {$Y^*$};
            \node[circ, text=gray, draw=gray, fill=white] at (1.6, 2.4) (Z1) {$Z^*$};
    
            \draw[arrow] (UX) edge [bend left=0] (X);
            \draw[arrow] (UY) edge [bend left=0] (Y);
            \draw[arrow] (UZ) edge [bend left=0] (Z);
            \draw[arrow] (UY) edge [bend left=0, gray] (Y1);
            \draw[arrow] (UZ) edge [bend left=0, gray] (Z1);
            \draw[arrow] (X) edge [bend left=0] (Y);
            \draw[arrow] (Z) edge [bend left=0] (X);
            \draw[arrow] (X1) edge [bend left=0, gray] (Y1);

            \node at (0.8,-0.9) {\normalsize (c)};
        \end{scope}

        \begin{scope}[xshift=210]
            \node[circ, text=gray, draw=gray, fill=white] at (0, 1.2) (X) {$X$};
            \node[circ, text=gray, draw=gray, fill=white] at (0, 0) (Y) {$Y$};
            \node[circ, text=gray, draw=gray, fill=white] at (0, 2.4) (Z) {$Z$};
            \node[circ, text=black, draw=black, fill=white, dashed] at (0.8, 1.8) (UX) {$U_X$};
            \node[circ, text=black, draw=black, fill=white, dashed] at (0.8, 0.6) (UY) {$U_Y$};
            \node[circ, text=black, draw=black, fill=white, dashed] at (0.8, 3.0) (UZ) {$U_Z$};
            \node[circ, text=black, draw=black, fill=white] at (1.6, 1.2) (X1) {$f(y)$};
            \node[circ, text=black, draw=black, fill=white] at (1.6, 0) (Y1) {$Y^\star$};
            \node[circ, text=black, draw=black, fill=white] at (1.6, 2.4) (Z1) {$Z^\star$};
    
            \draw[arrow] (UX) edge [bend left=0, gray] (X);
            \draw[arrow] (UY) edge [bend left=0, gray] (Y);
            \draw[arrow] (UZ) edge [bend left=0, gray] (Z);
            \draw[arrow] (UY) edge [bend left=0] (Y1);
            \draw[arrow] (UZ) edge [bend left=0] (Z1);
            \draw[arrow,gray] (X) edge [bend left=0] (Y);
            \draw[arrow,gray] (Z) edge [bend left=0] (X);
            \draw[arrow] (X1) edge [bend left=0] (Y1);
            \draw[arrow] (Y) edge [bend left=40] (X1);

            \node at (0.8,-0.9) {\normalsize (d)};
        \end{scope}
        
        \end{tikzpicture}
    \caption{(a) The DAG for a generic SCM $\scm$; solid nodes represent endogenous variables, dashed nodes represent exogenous variables. 
    (b) The DAG after intervention $\dointv{X}{1}$. 
    (c) The twin graph for the counterfactual $\ctf{X}{1}{X}{0}$; black nodes represent actual variables, gray nodes represent counterfactual variables.
    (d) The twin graph for the intentional intervention $\iiintv{X^\star}{f(y)}$.
    }
    \label{fig:SCM-graphs}
\end{figure}
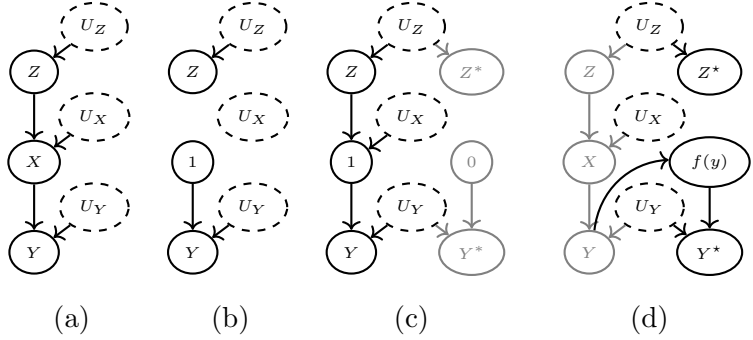

\begin{mydef}[SCM \cite{pearl2009causality}] A structural causal model (SCM) $\scm$ is a tuple $\scmsignature$ where $\envars = \setNelems{X}$ is a set of $N$ endogenous variables representing quantities relevant to the modeler; $\exvars = \setNelems{U}$ is a set of $N$ exogenous variables capturing variability outside the model; $\structfuncs = \setNelems{f}$ is a set of $N$ structural functions, one for each endogenous node, such that the value of $X_i$ is computed as $X_i = f_i(\envars\setminus X_i, U_i)$; $\probdists = \setNelems{P}$ is a set of $N$ probability distributions, one for each exogenous node. 
\end{mydef}

An SCM $\scm$ immediately induces a directed graph $\scmdag = \tuple{\mathbf{V}, \mathbf{E}}$ over vertices $\mathbf{V}=\envars$ and edges $\mathbf{E}$ such that $(X_i,X_j) \in \mathbf{E}$ if $X_i$ is among the arguments of $f_j$; see, for instance, \Cref{fig:SCM-graphs}(a). We will assume that all our SCMs induce directed acyclic graphs (DAG); we will refer to the graph-theoretic parents of a node $X$ as $\parents(X)$ and to graph-theoretic descendants as $\desc(X)$. 
For the sake of simplicity, we also assume that all the distributions in $\probdists$ are independent from each other; together with acyclicity, this implies that our SCMs are Markovian \citep{pearl2009causality}. 
SCMs allow us to model experiments in the form of interventions. For the sake of simplicity, and in line with the basic definition of intentional intervention we will provide later, we limit ourselves to operators on a single variable (for more general definitions see \Cref{app:notation}) :

\begin{mydef}[(Perfect) Intervention \citep{pearl2009causality}]
Given an SCM $\scm$, a variable ${X_i} \in \envars$ and value ${x}$, an intervention $\dointv{{X_i}}{{x}}$ is an operator applied to $\scm$ generating a new post-intervention SCM $\scmu{\doopt} = \myset{\envars,\exvars,\structfuncsu{\doopt},\probdists}$ where $\structfuncsu{\doopt} = \structfuncs$ except for $f_i$ being replaced by the constant $x$.  
\end{mydef}

An intervention has both an algebraic aspect (shown in the definition as the substitution of a structural function with a constant), and a graphical aspect (due to the fact that $\scmu{\doopt}$ induces a new DAG where the edges incoming into intervened node ${X}$ are removed, as in \Cref{fig:SCM-graphs}(b)). 
A perfect intervention can be generalized to a mechanism change:

\begin{mydef}[Mechanism Change \citep{tian2013causal}]
Given an SCM $\scm$, a variable ${X_i} \in \envars$ and a function ${g}$, a mechanism change $\dointv{{X_i}}{{g}}$ is an operator applied to $\scm$ generating a new post-intervention SCM $\scmu{\doopt} = \myset{\envars,\exvars,\structfuncsu{\doopt},\probdists}$ where $\structfuncsu{\doopt} = \structfuncs$ except for $f_i$ being replaced by the function $g$.
\end{mydef}

A well-behaved mechanism change preserves or reduces the signature of the intervened function, thus inducing a new model $\scmu{\doopt}$ according to an analogous 
procedure as for perfect interventions.
Finally, SCMs allow also for the evaluation of counterfactual statements:

\begin{mydef}[Counterfactual \citep{pearl2009causality}]
Given an SCM $\scm$, a variable ${X} \in \envars$ and values ${x,x^*}$ for ${X}$, a counterfactual $\ctf{{X}}{{x}}{{X}}{{x^*}}$ is an operator applied to $\scm$ and generating a new counterfactual SCM $\ctfscm = \ctfscmsignature$ via a two-step procedure:
\begin{enumerate}
    \item \emph{Abduction:} infer the distributions $\probdistsu{\vert {X}={x}}$ of the exogenous variables $\exvars$ in $\scm$ under the conditioning ${X}={x}$, and define a new SCM $\scmu{\abd} = \myset{\envars,\exvars,\structfuncs,\probdistsu{\vert {X}={x}}}$.
    
    \item \emph{Intervention}: perform intervention $\dointv{{X}}{{x^*}}$ on $\scmu{\abd}$, generating the SCM $\ctfscm$.
\end{enumerate}

\end{mydef}

Besides the algorithmic procedure, the model induced by a counterfactual can also be given a graphical representation. For the sake of simplicity, we consider \emph{twin graphs} \citep{balke2022probabilistic} generated through a twinning operation \citep{bongers2021foundations}:

\begin{mydef}[Twin model]\label{def:twin_model}
    Given an SCM $\scm$ and counterfactual $\ctf{{X}}{{x}}{{X}}{{x^*}}$, a twin model is a SCM $\scmu{\twin}$ constructed via a two-step procedure:
    
    \begin{enumerate}
        \item \emph{Replica:} infer the distribution $\probdistsu{\vert {X}={x}}$ in $\scm$ under the conditioning ${X}={x}$ and define a new model $\scmu{\rep}=\myset{\envars \cup \envars^*, \exvars, \structfuncs\cup\structfuncs^*, \probdistsu{\vert {X}={x}}}$   where $\envars^*, \structfuncs^*$ replicates $\envars, \structfuncs$.

        \item \emph{Intervention:} perform intervention $\dointv{{X}^*}{{x^*}}$ on $\scmu{\rep}$, generating the SCM $\scmu{\twin}$.
    \end{enumerate}
       
\end{mydef}

Intuitively, a twin model for $\ctf{{X}}{{x}}{{X}}{{x^*}}$ is constructed following the definition of counterfactuals: it joins $\scmdagu{}$ and its replica on the same exogenous variables $\mathcal{U}$ while distinguishing actual variables $\mathcal{X}$ and counterfactual variables $\mathcal{X}^*$, as illustrated in \Cref{fig:SCM-graphs}(c). 

\section{Desiderata for intention modeling} \label{sec:Desiderata}

In this section we express our desiderata for modeling  the behavior of an intentional agent (see \Cref{app:desiderata} for a more thorough justification):

\begin{itemize}
    \item[(D1)] \rebuttal{\textit{Causal grounding}: intentional interventions should be represented with an operator relating an unintervened causal model and an intervened one, so that the unintervened model (i) can be described and tested causally, independently of any agent, and (ii) can support intentional interventions upon different variables and with different aims, so discriminating between observational and interventional quantities;}



    \item[(D2)] \emph{State-awareness}: it must be explicit that the agent's behavior is derived from an awareness of some state expressed within the model;

    \item[(D3)] \emph{Goal-directedness}: it must be explicit that the agent's behavior is aimed at achieving a desired state expressed within the model;

    \item[(D4)] \textit{Explaining data}:
    models should help us make sense of data which are otherwise unexplained by a given causal model,
    as in \Cref{ex:heating_causal};

    \item[(D5)] \emph{Discriminating intentions}: models should help us isolate the intended effects of an agent among all the effects obtained by an intervention, as in \Cref{ex:smoking_causal};

    \item[(D6)] \emph{Acyclicity}: models should agree with the acyclicity assumption of standard SCMs; 

    \item[(D7)] \emph{Time-agnostic}: models should allow us to work with data with no explicit time dimension, as in \Cref{ex:heating_causal,ex:smoking_causal};
        
\end{itemize}

\section{Previous approaches to intention modeling}

In this section we review previous approaches from the literature to modeling agents with SCMs in light of the desiderata listed above (see \Cref{tab:desiderata} for a summary). 

\begin{table*}[]
    \centering
    \begin{tabular}{cccccc}
\hline 
\textbf{\emph{\footnotesize{}Desiderata}} & 
\textbf{\emph{\footnotesize{}Exogenous}} & \textbf{\emph{\footnotesize{}Endogenous}} & \textbf{\emph{\footnotesize{}Time-based}} & \textbf{\emph{\footnotesize{}Final}}\tabularnewline
\hline 
{\footnotesize{}(D1) Causal grounding} & 
{\footnotesize{}$\times$} & {\footnotesize{}$\times$} & {\footnotesize{}$\times$}& {\footnotesize{}$\checkmark$} & \tabularnewline
\hline 
{\footnotesize{}(D2) State-awareness} & 
{\footnotesize{}$\times$} & {\footnotesize{}$\checkmark$} & {\footnotesize{}$\checkmark$} & {\footnotesize{}$\checkmark$} & \tabularnewline
\hline 
{\footnotesize{}(D3) Goal-directedness} & 
{\footnotesize{}$\times$} & {\footnotesize{}$\checkmark$} & {\footnotesize{}$\checkmark$} & {\footnotesize{}$\checkmark$}& \tabularnewline
\hline 
{\footnotesize{}(D4) Explaining data} & 
{\footnotesize{}$\times$} & {\footnotesize{}$\checkmark$} & {\footnotesize{}$\checkmark$} & {\footnotesize{}$\checkmark$}& \tabularnewline
\hline 
{\footnotesize{}(D5) Discriminating intentions} & 
{\footnotesize{}$\times$} & {\footnotesize{}$\checkmark$} & {\footnotesize{}$\checkmark$} & {\footnotesize{}$\checkmark$}& \tabularnewline
\hline 
{\footnotesize{}(D6) Acyclicity} & 
{\footnotesize{}$\checkmark$} & {\footnotesize{}$\times$} & {\footnotesize{}$\checkmark$} & {\footnotesize{}$\checkmark$}& \tabularnewline
\hline 
{\footnotesize{}(D7) Time-agnostic} & 
{\footnotesize{}$\checkmark$} & {\footnotesize{}$\checkmark$} & {\footnotesize{}$\times$} & {\footnotesize{}$\checkmark$}& \tabularnewline
\hline 
\end{tabular}
    \caption{List of desiderata for intention modeling}
    \label{tab:desiderata}
\end{table*}

\narrative{Causal modeling of action.}
An SCM represents a \emph{system in isolation} with a set of deterministic endogenous variables as well as a set of latent stochastic exogenous variables. This corresponds to a level of abstraction at which \dd{nature can be carved at its joints}, meaning that we establish a boundary such that the set of endogenous variables behaves as a closed mechanistic system, while the exogenous variables amount to stochastic processes independent from the system under study.
However, an intentional agent breaks this partitioning, as it is not part of the modeled system but still intervenes on it 
\citep{vonwright1971}. In order to preserve the setup of systems in isolation, agents can be factored within or without an SCM.

\paragraph{Factoring intention into exogenous variables.}
Factoring an agent outside an SCM means relegating it into the exogenous variables $\exvars$; probability distributions $\probdists$ over $\exvars$ may then approximate the intentions of an agent, as in the following example. 

\begin{example} \label{ex:exogenous_modelling}
Reconsider the heating scenario of \Cref{ex:heating_causal} and suppose we were to know that agents intentionally turn on the heater ($H$) when the weather ($W$) is bad. We could then set the heater exogenous variable ($U_H$) to match the frequency of bad weather, in an attempt to reproduce observed data, as done in \Cref{fig:basic-solutions-heating-system}(a).   
\end{example}

This solution has severe limitations. First, the chosen exogenous distributions may happen to bias the SCM towards outcomes resembling the behavior of a system controlled by an agent, but it does not explicitly model intentional interventions (D1). State-awareness and goal-directedness are absent (D2, D3). The model fails to explain explicitly the dependencies between variables induced by the agent (D4) and we are prevented from reasoning about the agents’ intentions (D5).
While the former points might be addressed by considering semi-Markovian models \citep{pearl2009causality}, the latter depends on the nature of exogenous variables.
Thus, only properties intrinsic to SCMs (D6, D7) are retained.

\begin{figure*}
    \centering
    \begin{tikzpicture}
        [node distance = 1cm, thick,
        arrow/.style = {draw, ->},
        circ/.style = {draw, ellipse}]

        \begin{scope}[xshift=0]
            \node[circ, text=black, draw=black, fill=white] at (-1.2, 0.5) (W) {$W$};
            \node[circ, text=black, draw=black, fill=white] at (0, 0) (T) {$T$};
            \node[circ, text=black, draw=black, fill=white] at (1.2, 0.5) (H) {$H$};
            \node[circ, text=black, draw=black, fill=white, dashed] at (1.2, 1.8) (UH) {$U_H$};
    
            \draw[arrow] (W) edge [bend left=0] (T);
            \draw[arrow] (H) edge [bend left=0] (T);
            \draw[arrow] (UH) edge [bend left=0] (H);

            \node at (0,-1.0) {(a)};
        \end{scope}

        \begin{scope}[xshift=120]
            \node[circ, text=black, draw=black, fill=white] at (-1.2, 0.5) (W) {$W$};
            \node[circ, text=black, draw=black, fill=white] at (0, 0) (T) {$T$};
            \node[circ, text=black, draw=black, fill=white] at (1.2, 0.5) (H) {$H$};
            \node[circ, text=black, draw=black, fill=white] at (0, 1.8) (I) {$I$};
    
            \draw[arrow] (W) edge [bend left=0] (T);
            \draw[arrow] (H) edge [bend left=0] (T);
            \draw[arrow] (I) edge [bend left=0] (H);
            \draw[arrow] (T) edge [bend left=0] (I);

            \node at (0,-1.0) {(b)};
        \end{scope}

        \begin{scope}[xshift=240]
            \node[circ, text=black, draw=black, fill=white] at (-1.2, 0.5) (W0) {$W_0$};
            \node[circ, text=black, draw=black, fill=white] at (0, 0) (T0) {$T_0$};
            \node[circ, text=black, draw=black, fill=white] at (1.2, 0.5) (H0) {$H_0$};

            \node[circ, text=black, draw=black, fill=white] at (2.8, 0.5) (W1) {$W_1$};
            \node[circ, text=black, draw=black, fill=white] at (4, 0) (T1) {$T_1$};
            \node[circ, text=black, draw=black, fill=white] at (5.2, 0.5) (H1) {$H_1$};
            
            \node[circ, text=black, draw=black, fill=white] at (4, 1.8) (I_1) {$I_1$};
    
            \draw[arrow] (W0) edge [bend left=0] (T0);
            \draw[arrow] (H0) edge [bend left=0] (T0);

            \draw[arrow] (W1) edge [bend left=0] (T1);
            \draw[arrow] (H1) edge [bend left=0] (T1);
            
            \draw[arrow] (W0) edge [bend left=30] (W1);
            \draw[arrow] (H0) edge [bend left=30] (H1);
            \draw[arrow] (T0) edge [bend left=0] (T1);
            
            \draw[arrow] (I_1) edge [bend left=15] (H1);
            \draw[arrow] (T0) edge [bend left=30] (I_1);

            \node at (2,-1.0) {(c)};
        \end{scope}
        
        \end{tikzpicture}
    \caption{Existing approaches to account for agents' intentions: (a) factoring into exogenous variables; (b) factoring into endogenous variables; (c) time-based models.}
    \label{fig:basic-solutions-heating-system}
\end{figure*}
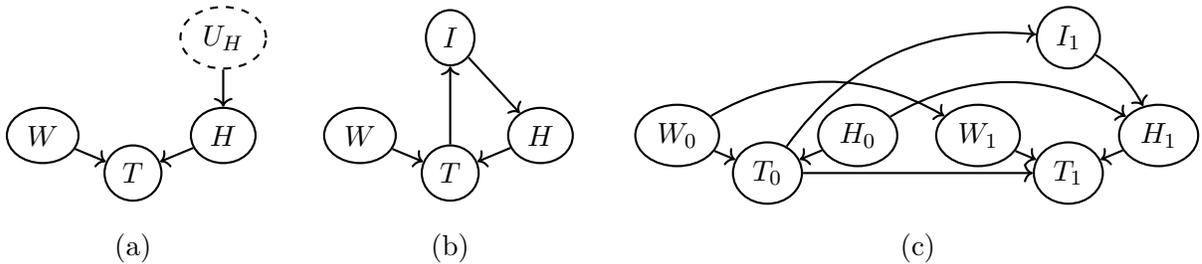

\paragraph{Factoring intention into endogenous variables.}
Accounting for state-awareness and goal-directedness (D2, D3) can be achieved with the more common approach of factoring the agents' intentions. into a new endogenous variable $I\in\envars$, as in the following example.

\begin{example} \label{ex:endogenous_modelling}
In the heating scenario of \Cref{ex:heating_causal} suppose that agents assess the temperature of the room ($T$) and then decide whether to turn on the heater ($H$). We could then introduce a new variable ($I$) with such incoming and outgoing edges, as in \Cref{fig:basic-solutions-heating-system}(b).
\end{example}

Beyond relying on standard endogenous variables, intentional action might also be captured through a decision-theoretic model \citep{dawid2021decision} by extending the signature of an SCM to account for new types of variables, such as decision and utility \citep{ward2024reasons}.
These solutions allow for explaining new dependencies in the data introduced by the agent's interventions (D4) and for discriminating intentions (D5). However, the introduction of a variable for the intention, linking back some effects to the intervened variable, breaks the acyclicity assumption (D6) as in \Cref{fig:basic-solutions-heating-system}(b), opening the way to backward causation \citep{faye2025}. Moreover, an agent (pertaining to a different level of abstraction) is reduced to a mechanism; thus agent's interventions can not distinguished from autonomous causal relations (D1), erasing any difference in interpretation between observational quantities (arising from the system in the absence of any agent) and interventional quantities (arising from a system whose dynamics have been affected by an agent).

\paragraph{Factoring agents into the time dimension.}
The cyclicity shortcoming (D6) can be solved by explicitly introducing a time dimension that specifies how past effects inform current actions, as in the following example.

\begin{example}\label{ex:time_modelling}
In the heating scenario of \Cref{ex:heating_causal} we could model the intention of the agent at the current time ($I_1$), based on \dd{listening to} the past room temperature ($T_0$) in order to act on the heater at the current time ($H_1$), as in \Cref{fig:basic-solutions-heating-system}(c).
\end{example}

Extended summary graphs for time-based causal models \citep{assaad2023identifiability} or causal representations of Markov decision processes in reinforcement learning \citep{buesing2018woulda,zeng2024survey} allow for the modeling of intentional relations (D2, D3, D4, D5) at the cost of introducing new variables and observing the system at multiple time-steps (D7). The action of an agent can be hard-coded in the form of fixed policies or dynamic treatments within the agent variables \citep{murphy2003optimal}, although this will not provide a flexible modeling capable to account for different actions performed on the same causal system (D1), again erasing the distinction between observational and interventional quantities.

\section{A new approach to intention modeling}

\narrative{Outline of our approach.} In this section, we propose a new modeling approach that satisfies all our desiderata. We will introduce some key assumptions for intentional action, formalize the concepts of intentional intervention and SFM, and discuss some of their properties.

\subsection{Assumptions for intention modeling}\label{ssec:assumptions}

In order to act intentionally, the agent must (i) have knowledge of the causal system it is interacting with, and (ii) intervene without the state of the world significantly changing.

\begin{assumption}[Perfect knowledge]\label{ass:perfect-knowledge}
    The intentional agent knows the true causal structure (DAG) of the causal system it is interacting with.
\end{assumption}

\narrative{Meaning of the assumptions: Perfect knowledge}
\Cref{ass:perfect-knowledge} can be easily justified for interventions upon simple causal systems, such as \Cref{ex:heating_causal}, where cognitive and cultural variation does not have a significant impact on the agents' behavior. If we were to study interventions upon more complex systems or those whose functioning is open to controversy, this assumption would have to be relaxed, taking into account the epistemic state of the agent \citep{friedenberg2025intents}. \rebuttal{Dropping this assumption would then allow us to model scenarios where the agent might misperceive the environment or misspecify its causal model (see \Cref{app:definition})}.

\begin{assumption}[High-frequency agent]\label{ass:high-frequency}
    The intentional agent acts at a significantly higher frequency than the exogenous variables.
\end{assumption}

\narrative{Meaning of the assumptions: High frequency}
\Cref{ass:high-frequency} reflects the case in \Cref{ex:heating_causal}, where the weather does not change significantly before and after intervention. If this assumption were not satisfied, the exogenous variables would change quickly enough that, by the time the agent manages to turn on the heater, room temperature could already be warm.
This assumption provides a measure of stability necessary for state-awareness and for actions to be purposeful. 
Notice, however, that the assumption needs to hold only at a chosen level of abstraction at least approximately: it might not be true that the exogenous variables are exactly fixed, but should at least be stable enough to be treated as constant. 
Violation of this requirement would lead to time-based models, with exogenous variables resampled at each timestep.

\subsection{Intentional interventions and structural final models}\label{ssec:intentional_interventions}

\narrative{Intuition for our approach.}
From a teleological perspective, \rebuttal{under \Cref{ass:perfect-knowledge}}, what determines the value of an intervened variable are not its \textit{factual ancestors} (what its causes are) but its \textit{counterfactual descendants} (what the outcome would be without intervention).
Intentional interventions thus bear a resemblance both to standard interventions and to counterfactuals:
\begin{itemize}
    \item[(i)] In standard interventions, intervened variables are assigned arbitrary values and thus freed from all their parents, if any. In intentional interventions, intervened variables similarly lose incoming arrows because of the control exercised on them by the agent, but also gain new incoming arrows encoding the agent's intention;
    \item[(ii)] In counterfactuals, an observed factual SCM is twinned with an imagined counterfactual SCM. In intentional interventions, a counterfactual unintervened SCM is twinned with a factual intervened SCM.
\end{itemize}
Let us formalize the application of an intentional intervention as an operator generating a SFM, restricting our attention to an intentional intervention on a single variable ($X$) and having another single variable as goal ($Y$):

\begin{mydef}[Intentional Intervention and Structural Final Model] \label{def:intentional_intervention}
Given an SCM $\scm$, an intervention variable $X\in \envars$, a variable $Y \in \envars$ such that $Y \in \desc(X)$, and a function $f(y)$, an \emph{intentional intervention} $\iiintv{X^\star}{{f}(y)}$ is an operator applied to $\scm$ and generating a new SCM, \rebuttal{denoted as \emph{structural final model (SFM)}} $\finscm$, constructed via a two-step procedure:

    \begin{enumerate}
        \item \emph{Replica:} infer the distribution ${\probdistsu{}}_{\vert Y=y}$ in $\scm$ under the conditioning $Y=y$ and define a new model $\scmu{\rep} = \myset{\envars \cup \envars^\star, \exvars, \structfuncs\cup\structfuncs^\star, {\probdistsu{}}_{\vert Y=y}}$ where $\envars^\star, \structfuncs^\star$ replicate $\envars,\structfuncs$.

        \item \emph{Intervention:} perform intervention $\dointv{X^\star}{f(y)}$ on $\scmu{\rep}$, generating SCM $\finscm$.
    \end{enumerate}
\end{mydef}


%
%

\narrative{Interpretation}
The first step \rebuttal{operates analogously to the first step of \Cref{def:twin_model} by twinning} two models, the original causal $\scm$ and the replica $\scmu{rep}$ which the agent intervenes upon; the second step performs the actual intervention which 
is encoded by the mechanism change $f(y)$ as in \Cref{fig:SCM-graphs}(d). Thus, SFMs avoid representing intentions as model variables and instead express them through a \textit{relationship} among two causal models.

Notice that SFMs assume that what is factual, and observed by researchers, is the result of an intervention, while what is counterfactual is the grounding causal system without any intervention. Graphically, in the case of counterfactuals as in \Cref{fig:SCM-graphs}(c), given the twin model $\scmu{\twin}$, it is assumed that an investigator observes data from $\envars$ (factual black variables) while $\envars^*$ are unobservable (counterfactual gray variables). In the case of intentional interventions as in \Cref{fig:SCM-graphs}(d), instead, given the model $\finscm$, an investigator observes data from $\envars^\star$ (intentionally intervened factual black variables) but not from $\envars$ (unintervened counterfactual gray variables).
\rebuttal{This inversion might be interpreted as the act of an agent that tries to \emph{anticipate the functions of a system} by counterfactually simulating their evolution and then acting to steer they dynamics towards its desired outcome.}

\narrative{Role of assumptions}
\Cref{ass:perfect-knowledge} guarantees that the intentions of the agent are consistent with the workings of the causal system, that is, the mechanism change $f$ can be purposefully used to obtain the intended outcome on $Y^\star$. \Cref{ass:high-frequency} justifies the twinning step. 
Importantly, notice that we rely on a deterministic function $f$, but we do not require it to be optimal or the agent to be moral or rational (\emph{contra} \cite{halpernkleimanweiner2018}), allowing us to capture the behavior of agents with limited computational capacity, affected by biases, or who are different from us (see \Cref{app:definition}).

\subsection{Properties of SFMs and intentional interventions}

We now prove a few properties of SFMs; formal proofs are in \Cref{app:proofs}.

\begin{myprop}\label{prop:DAG}
    SFMs imply DAGs.
\end{myprop}

The relation with standard interventions and counterfactuals hinted in \Cref{ssec:intentional_interventions} can now be formalized as follows:

\begin{mylemma}\label{lem:counterfactual}
    A counterfactual $\ctf{X}{x}{X}{x^*}$ is an intentional intervention $\iiintv{X^\star}{f(y)}$ for which: (i) \Cref{ass:high-frequency} holds perfectly; (ii) $Y=X$; (iii) $f(y)=x^*$.
\end{mylemma}

A counterfactual is a limit-case intentional intervention where the agent has state-awareness and goal-directedness about a single variable, and acts directly on it. Since perfect interventions can be expressed with counterfactual notation, it is immediate to show that perfect interventions too can be subsumed by intentional interventions: 

\begin{mylemma}\label{lem:intervention}
    A perfect intervention $\dointv{X}{x}$ is an intentional intervention $\iiintv{X^\star}{f(y)}$ for which: (i) \Cref{ass:high-frequency} does not necessarily hold; (ii) $Y=\emptyset$; (iii) $f(y)=x$.
\end{mylemma}

Thus, a perfect intervention is also a limit-case intentional intervention where the agent has no state-awareness nor goal-directedness, and acts arbitrarily.
In \Cref{app:proofs} (\Cref{lem:lemma3}) we offer a similar result for dynamic treatments \citep{murphy2003optimal}, \rebuttal{further highlighting the difference between our approach and time-based causal models.}

\section{Teleological inference}

SFMs allow us to formalize teleological questions as \textit{teleological inferences}. \Cref{ex:heating_causal} and \ref{ex:smoking_causal} provide the illustration of two basic graphical structures (a collider and a fork) which are related to two fundamental tasks of teleological inference: (i) detecting agents and (ii) discovering intentions. In \Cref{app:simulations} we provide an implementation and simulations of teleological inference for \Cref{ex:heating_causal} and \ref{ex:smoking_causal}.

\begin{figure}
    \centering
    \begin{tikzpicture}
        [node distance = 1cm, thick,
        arrow/.style = {draw, ->},
        circ/.style = {draw, ellipse}]

        \begin{scope}[xshift=0]
            \node[circ, text=gray, draw=gray, fill=white] at (-1.2, 0.5) (H) {$H$};
            \node[circ, text=gray, draw=gray, fill=white] at (0, 0) (T) {$T$};
            \node[circ, text=gray, draw=gray, fill=white] at (1.2, 0.5) (W) {$W$};
            
            \node[circ, text=black, draw=black, fill=white, dashed] at (1.2, 1.6) (UW) {$U_W$};
            \node[circ, text=black, draw=black, fill=white, dashed] at (0, -1.1) (UT) {$U_T$};
            \node[circ, text=black, draw=black, fill=white, dashed] at (-1.2, 1.6) (UH) {$U_H$};
    
            \node[circ, text=black, draw=black, fill=white] at (3, 0.5) (H') {$H^\star$};
            \node[circ, text=black, draw=black, fill=white] at (4.2, 0) (T') {$T^\star$};
            \node[circ, text=black, draw=black, fill=white] at (5.4, 0.5) (W') {$W^\star$};        
            
            \draw[arrow] (W) edge [bend left=0, gray] (T);
            \draw[arrow] (H) edge [bend left=0, gray] (T);
    
            \draw[arrow] (W') edge [bend left=0] (T');
            \draw[arrow] (H') edge [bend left=0] (T');
            
            \draw[arrow] (UW) edge [bend left=0, gray] (W);
            \draw[arrow] (UH) edge [bend left=0, gray] (H);
            \draw[arrow] (UT) edge [bend left=0, gray] (T);
            
            \draw[arrow] (UW) edge [bend left=10] (W');
            \draw[arrow] (UT) edge [bend left=-10] (T');
            
            \draw[arrow] (T) edge [bend left=-10] (H');

            \node at (2,-2.0) {(a)};
        \end{scope}

        \begin{scope}[xshift=240, yshift=15]
            \node[circ, text=gray, draw=gray, fill=white] at (-1.2, -0.5) (D) {$D$};
            \node[circ, text=gray, draw=gray, fill=white] at (0, 0) (S) {$S$};
            \node[circ, text=gray, draw=gray, fill=white] at (1.2, -0.5) (P) {$P$};
            
            \node[circ, text=black, draw=black, fill=white, dashed] at (-1.2, -1.6) (UD) {$U_D$};
            \node[circ, text=black, draw=black, fill=white, dashed] at (0, 1.1) (US) {$U_S$};
            \node[circ, text=black, draw=black, fill=white, dashed] at (1.2, -1.6) (UP) {$U_P$};
    
            \node[circ, text=black, draw=black, fill=white] at (3, -0.5) (D') {$D^\star$};
            \node[circ, text=black, draw=black, fill=white] at (4.2, 0) (S') {$S^\star$};
            \node[circ, text=black, draw=black, fill=white] at (5.4, -0.5) (P') {$P^\star$};        
            
            \draw[arrow] (S) edge [bend left=0, gray] (P);
            \draw[arrow] (S) edge [bend left=0, gray] (D);
    
            \draw[arrow] (S') edge [bend left=0] (P');
            \draw[arrow] (S') edge [bend left=0] (D');
            
            \draw[arrow] (US) edge [bend left=0, gray] (S);
            \draw[arrow] (UP) edge [bend left=0, gray] (P);
            \draw[arrow] (UD) edge [bend left=0, gray] (D);
            
            \draw[arrow] (UP) edge [bend left=-10] (P');
            \draw[arrow] (UD) edge [bend left=-5] (D');
            
            \draw[arrow] (P) edge [bend left=10] (S');

            \node at (2,-2.5) {(b)};
        \end{scope}

        \end{tikzpicture}
    \caption{(a) SFM for heating. (b) SFM for smoking.}
    \label{fig:final-systems}
\end{figure}
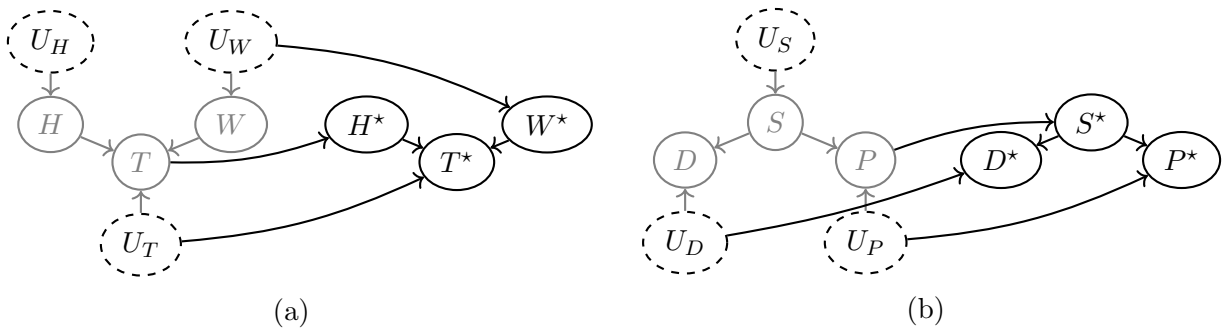

\subsection{Detecting agents}
Consider the problem of determining whether the data an investigator collects from a system is affected by the intentional intervention of an agent, as in the following example:

\begin{example}
Reconsider the heating scenario of \Cref{ex:heating_causal} and suppose we observe data as in table in \Cref{fig:basic-heating-system}(a). Can we decide whether the data comes from the unintervened causal system or instead an agent has intentionally acted on some variable?
\end{example}

If we observe from data a dependency between $W$ and $H$ which we do not expect from the graphical model we have two options: either we believe the model to be wrong, and so try to correct it by inserting a causal dependency between $W$ and $H$; or, assuming Markovianity is indeed justified, we interpret this violation of expected independencies as the sign of an intentional action.

If an agent is at work, it is \dd{controlling for a collider}, that is, observing a common effect and acting consequently \citep{compagno2025}. This means that we are not observing $W$ and $H$, but instead $W^\star$ and $H^\star$, that is, values which are the result of an intervention. The SFM in \Cref{fig:final-systems}(a) accounts for the agent's intentional intervention $\iiintv{H^\star}{{f}(t)}$ and graphically explains the generated dependency between $W^\star$ and $H^\star$ via a backdoor path across the twin models.

Importantly, the observed violation of Markovianity is not intrinsic to the causal model, as it would hold with data generated by the unintervened causal system; it is the intervention that leads to this particular violation. 
An SFM twins an unvarying causal model with the specific intentional intervention taking place on it. Different intentional interventions would tie the same underlying SCM with different twins, each accounting for specific expected independencies.
The problem of detecting agents can be generalized as:

\begin{problem}[Detecting agents]
Given a causal DAG $\scmdag$ and data $\mathcal{D}$ generated from the modeled system, is it possible to decide whether an intentional intervention happened?
\end{problem}

While graph discovery can be framed as a \emph{search problem} (find a solution for a given problem instance $P$), detecting agents can be conceived as a \emph{verification problem} (given problem instance $P$, verify if solution $s$ actually solves it).
See \Cref{app:simulations} for evidence that the intentional intervention $\iiintv{H^\star}{{f}(t)}$ can be detected from simulated data, by contrasting its expected independencies with those of the unintervened SCM.
A first identifiability result follows immediately by construction from \Cref{ex:heating_causal}:

\begin{myprop}[Identifiability for detecting agents in simple colliders]\label{prop:detecting_agents}
    Given an instance of the problem $\langle\scmdag,\mathcal{D}\rangle$ of detecting an agent in a simple collider, if the agent intervenes on one of the parents $\parents(C)$ of the collider $C$, an intentional intervention can be detected from $\mathcal{D}$ as a violation of Markovianity in $\scmdag$.
\end{myprop}

\narrative{Identifiability}
We suggest that algorithms based on evaluating violation of Markovianity might be sound, although incomplete, methods for detecting agents; violations of Markovianity are sufficient to flag the presence of an intentional intervention, but their absence can not rule out all intentional intervention. We leave to future work general identifiability criteria and the derivation of a potentially sound and complete algorithm for detecting agents.

\subsection{Intention discovery}

Assuming an intentional intervention, we might be interested in deciding what variables are \dd{listened to} by the agent, as in the following example:

\begin{example}
Reconsider the smoking scenario of \Cref{ex:smoking_causal} and suppose that data is generated from an intentional intervention. Can we know whether people smoke in order to obtain pleasure or instead \rebuttal{tar deposits}?
\end{example}

In the example above, a smoking agent might have as an aim either variable $P^\star$ or $D^\star$ (or possibly both, against \cite{davidson1963}). The fork structure in \Cref{fig:final-systems}(b) suggests a clear experimental procedure to assess intentional intervention: if the goal of the agent is only pleasure, then intervening on $P$ will affect $S^\star$ while acting on $D$ will leave $S^\star$ unmodified.
More generally, the problem of intention discovery can be expressed as:

\begin{problem}[Intention discovery]
Given the SCM $\scm$ and assuming an intentional intervention $\iiintv{X^\star}{{f}(y)}$ on $X^\star$, is it possible to identify the argument variable $Y$ via observational or experimental data?
\end{problem}

Intention discovery might be seen as an \emph{experimental design} problem supported by the use of SFMs to decide which interventions to perform in order to assess the sensitivity of the agent's function $f$ to different goals. 
See \Cref{app:simulations} for evidence that the intentional intervention $\iiintv{S^\star}{{f}(p)}$ can be discovered from simulated data, by contrasting its expected distribution with that of $\iiintv{S^\star}{{f}(d)}$.  
As before, the construction in \Cref{ex:smoking_causal} suggests an immediate protocol to falsify teleological statements:

\begin{myprop}[Identifiability for intention discovery in simple forks]\label{prop:discovering_intentions}
    Given an instance of the problem $\langle \scm,\finscm \rangle$ of intention discovery in a simple fork, the argument $y$ of the intentional intervention $\iiintv{X^\star}{{f}(y)}$ can be identified by intervening on the descendants of $X$.
\end{myprop}

\narrative{Identifiability}
The proposition hints to a simple, albeit not computationally-scalable, algorithmic approach based on the testing of all combinations of descendants. Future work will look into more general and efficient algorithms capable to discriminate between intended and unintended consequences of interventions.


\section{Conclusion}

This paper shows the importance of taking interventions as objects of modeling.
Having presented the limitations of existing approaches, we introduced a novel formalism extending standard SCMs,
complying with the core desiderata for intention modeling (see \Cref{tab:desiderata}): intentional intervention is a flexible operator, capable of generating different structural final models (SFMs) from a common causal ground model, separating observational and interventional quantities (D1); state-awareness and goal-directedness are expressed in its arguments (D2-D3); unexpected dependencies can be explained without modifying our causal ground model, and may be used for detecting agents (D4); intentions can be discriminated empirically (D5); acyclicity is proved (D6); no time dimension is needed (D7).

Our approach is based upon common assumptions in causality; 
some assumptions might be dropped in future work, for the sake of generalization: Markovian models might be extended to semi-Markovian; single-variate operators might be extended to multi-variate ones, while single target variables might be expanded to multiple targets; twin graphs might be replaced by counterfactual graphs \citep{shpitser2007counterfactuals} or ancestral multi-world network \citep{correa2025counterfactual}. The assumption of perfect causal knowledge on the side of the agent might also be discarded, opening the door to teleological reasoning in presence of model misspecification.


\bibliography{bib}

\newpage

\appendix

\section{Notation and Definitions} \label{app:notation}

In this section we first provide some justification for our notation, and then we offer a few relevant graph-theoretical definitions, and more generic definition for the causal object and operators we have dealt with in the main paper.

\subsection{Discussion on notation}

\paragraph{Intervention and counterfactual operators.}
In this paper, we adopt an explicit notation for causal operators, such as the intervention operator $\dointv{X}{x}$ or counterfactual operator $\ctf{X}{x}{X}{x^*}$. Although a counterfactual quantity does not need an explicit operator (since the clashing assignments already hint to the fact that we are operating in different worlds), we believe that an explicit notation makes the discussion and the connection to intentional interventions clearer.

\paragraph{Intentional intervention operator.} For the intentional intervention operator we have opted for a notation that stresses the connection of our operator to intervention and counterfactual. The use of the keyword \emph{$\doopt$} is meant to highlight the foundational interventional nature of the operator; however, we decorate the keyword with a star $\star$ to point out the difference from standard perfect interventions. Similarly, the use of a superscript is meant to underline a connection with counterfactuals; however, we opted for a star $\star$ instead of an asterisk $*$ to preserve the distinction.

\subsection{Graph-theoretical definitions.}

We provide here a precise definition of graph-theoretical concepts referred to in the main paper.

\begin{mydef}[Parents]
    Let $\langle \mathcal{V}, \mathcal{E} \rangle$ be a directed graph defined over vertices $\mathcal{V}$ and edges $\mathcal{E}$. Given a vertex $V_i \in \mathcal{V}$, the parent set $\parents(V_i)$ of $V_i$ is the set of all vertices $V_j \in \mathcal{V}$ such that $(V_j,V_i) \in \mathcal{E}$.
\end{mydef}

\begin{mydef}[Directed path]
    Let $\langle \mathcal{V}, \mathcal{E} \rangle$ be a directed graph defined over vertices $\mathcal{V}$ and edges $\mathcal{E}$. A directed path $p$ is an ordered set of vertices $(V_1,...,V_n)$, for $n\geq1$, $V_i \in \mathcal{V}$ and $(V_i,V_{i+1}) \in \mathcal{E}$.
\end{mydef}

\begin{mydef}[Descendants]
    Let $\langle \mathcal{V}, \mathcal{E} \rangle$ be a directed graph defined over vertices $\mathcal{V}$ and edges $\mathcal{E}$. Given a vertex $V_i \in \mathcal{V}$, the set of descendants $\desc(V_i)$ of $V_i$ is the set of all vertices $V_j \in \mathcal{V}$ such that there exists a direct path $p$ from $V_i$ to $V_j$.
\end{mydef}

Notice that by definition, the set of descendants $\desc(V_i)$ of $V_i$ includes $V_i$ itself.

\subsection{Extended definitions}

We provide here more and generic description relevant to causal modelling. We start with the definition of dynamic SCM through the introduction of a temporal dimension.

\begin{mydef}[Dynamic SCM \cite{blondel2017identifiability}] A dynamic SCM $\scm$ is a tuple $\scmsignature$ where $\envars = \setNTelems{X}$ is a set of $N$ endogenous variables over $T$ timesteps representing quantities relevant to the modeler; $\exvars = \setNTelems{U}$ is a set of $N$ over $T$ timesteps exogenous variables capturing variability outside the model; $\structfuncs = \setNelems{f}$ is a set of $N$ structural functions, one for each endogenous node, such that the value of $X_i^{(t)}$ is computed as $X_i^{(t)} = f_i(\envars^{(\leq t)}\setminus X_i^{(t)}, U_i^{(t)})$, where $\envars^{(\leq t)}$ denotes the set of all the endogenous variable at a timestep less or equal than $t$; $\probdists = \setNelems{P}$ is a set of $N$ probability distributions, one for each exogenous node, such that the value of $U_i^{(t)}$ is computed as $U_i^{(t)} \sim P_i$.
\end{mydef}

Notice that the definition of a dynamic SCM assumes a fixed collection of endogenous and exogenous variables repeating in time. It also assumes stable structural functions and probability distribution. In particular, structural functions depend only on inputs from the past; for the structural equation not to change in time, every time-step depends on a number of inputs with a fixed delay in the past.

Next, we provide the standard definitions of intervention and mechanism change generalized to sets of variables.

\begin{mydef}[(Multi-variate Perfect) Intervention \citep{pearl2009causality}]
Given an SCM $\scm$, a set of variables $\mathbf{X} \in \envars$ and a set of values $\mathbf{x}$, one for each variable in $\mathbf{X}$, an intervention $\dointv{\mathbf{X}}{\mathbf{x}}$ is an operator applied to $\scm$ generating a new post-intervention SCM $\scmu{\doopt} = \myset{\envars,\exvars,\structfuncsu{\doopt},\probdists}$ where $\structfuncsu{\doopt}$ is the set of functions $f_i^{\doopt}$ such that $f_i^{\doopt} = f_i$ if $X_i \notin \mathbf{X}$, else $f_i^{\doopt} = x_i \in \mathbf{x}$.  
\end{mydef}

\begin{mydef}[(Multi-variate) Mechanism Change \citep{tian2013causal}]
Given an SCM $\scm$, a set of variables $\mathbf{X} \in \envars$ and a set of functions $\mathbf{g}$, one for each variable in $\mathbf{X}$, a mechanism change $\dointv{\mathbf{X}}{\mathbf{g}}$ is an operator applied to $\scm$ generating a new post-intervention SCM $\scmu{\doopt} = \myset{\envars,\exvars,\structfuncsu{\doopt},\probdists}$ where $\structfuncsu{\doopt}$ is the set of functions $f_i^{\doopt}$ such that $f_i^{\doopt} = f_i$ if $X_i \notin \mathbf{X}$, else $f_i^{\doopt} = g_i \in \mathbf{g}$.  
\end{mydef}

We give a definition of dynamic treatment \citep{murphy2003optimal,blondel2017identifiability,HernanRobins2020CausalInference} or dynamic plan \citep{pearl2009causality}, which formalizes conditional interventions on dynamic SCMs.

\begin{mydef}[Dynamic Treatment]
    Given an dynamic SCM $\scm$, a set of variables $\mathbf{X} \in \envars$ and a set of functions $\mathbf{g}(\mathbf{Z})$, one for each variable $X_i^{t}$ in $\mathbf{X}$ such that there is no $Z_i^{t'} \in \mathbf{Z}$ with $t'\geq t$, a dynamic treatment $\dointv{\mathbf{X}}{\mathbf{g}(\mathbf{Z})}$ is an operator applied to $\scm$ generating a new post-intervention dynamic SCM $\scmu{\doopt} = \myset{\envars,\exvars,\structfuncsu{\doopt},\probdists}$ where $\structfuncsu{\doopt}$ is the set of functions $f_i^{\doopt}$ such that $f_i^{\doopt} = f_i$ if $X_i \notin \mathbf{X}$, else $f_i^{\doopt} = g_i(\mathbf{Z}) \in \mathbf{g}$.  
\end{mydef}

Finally, we also offer a definition of counterfactual and twin model generalized to sets of variables.

\begin{mydef}[(Multi-variate) Counterfactual \citep{pearl2009causality}]
Given an SCM $\scm$, a set of variables $\mathbf{X} \in \envars$, and sets of values $\mathbf{x,x^*}$ for variable $\mathbf{X}$, a counterfactual $\ctf{\mathbf{X}}{\mathbf{x}}{\mathbf{X}}{\mathbf{x^*}}$ is an operator applied to $\scm$ and generating a new counterfactual SCM $\ctfscm = \ctfscmsignature$ via a two-step procedure:

\begin{enumerate}
    \item \emph{Abduction:} infer the distribution $\probdistsu{\vert \mathbf{X}=\mathbf{x}}$ of the exogenous variables $\exvars$ in $\scm$ under the conditioning $\mathbf{X}=\mathbf{x}$, and define a new SCM $\scmu{\abd} = \myset{\envars,\exvars,\structfuncs,\probdistsu{\vert \mathbf{X}=\mathbf{x}}}$.

    \item \emph{Intervention}: perform intervention $\dointv{\mathbf{X}}{\mathbf{x^*}}$ on $\scmu{\abd}$, generating the SCM $\ctfscm$.
\end{enumerate}
\end{mydef}

\begin{mydef}[(Multi-variate) Twin model] Given a counterfactual $\ctf{\mathbf{X}}{\mathbf{x}}{\mathbf{X}}{\mathbf{x^*}}$, a twin model is a SCM $\scmu{\twin}$ constructed via a two-step procedure:
    
    \begin{enumerate}
        \item \emph{Replica:} infer the distribution $\probdistsu{\vert \mathbf{X}=\mathbf{x}}$ in $\scm$ under the conditioning $\mathbf{X}=\mathbf{x}$ and define a new model $\scmu{\rep}=\myset{\envars \cup \envars^*, \exvars, \structfuncs\cup\structfuncs^*, \probdistsu{\vert \mathbf{X}=\mathbf{x}}}$   where $\envars^*, \structfuncs^*$ replicates $\envars, \structfuncs$.

        \item \emph{Intervention:} perform intervention $\dointv{\mathbf{X}^*}{\mathbf{x^*}}$ on $\scmu{\rep}$, generating the SCM $\scmu{\twin}$.
    \end{enumerate}
       
\end{mydef}



\section{Desiderata and assumptions for intention modelling}\label{app:desiderata}

In this section we provide discussion on our desiderata and assumptions.

\subsection{Desiderata for intention modelling}

We first discuss more at length the desiderata that we have chosen for intention modelling in \Cref{sec:Desiderata}.

\paragraph{(D1) Causal grounding.} \emph{intentional interventions should be represented with an operator relating an unintervened causal model and an intervened one, so that the unintervened model (i) can be described and tested causally, independently of any agent, and (ii) can support intentional interventions upon different variables and with different aims, so discriminating between observational and interventional quantities.}

\rebuttal{Formally, we want an intentional operator to be an operator $\mathcal{I}$ applicable to an SCM $\mathcal{M}$ and generating a new intentionally intervened model $\mathcal{I}(\mathcal{M})$, so that:
\begin{itemize}
    \item We can reason causally on $\mathcal{M}$ independently of $\mathcal{I}$;
    \item Given multiple intentional intervention $\mathcal{I},\mathcal{I'}$ we can compute multiple intentionally intervened models $\mathcal{I}(\mathcal{M})$, $\mathcal{I'}(\mathcal{M})$.
\end{itemize}}

This desideratum has two parts. The first part guarantees the existence of a causal model independently of an agent, in the same way a causal model exists independently of traditional interventions. The second part requires an intentional intervention to generate a new causal model formally related to the unintervened model; in other words, it states that an intentionally intervened model is not a new ad-hoc model created just to explain the activity of an agent, but instead it can be related both to the unintervened model and, potentially, to other causal models generated by other intentional interventions.

\paragraph{(D2) State-awareness.} \emph{It must be explicit that the agent's behavior is derived from an awareness of some state expressed within the model.}

\rebuttal{Formally, we want an intentional operator $\mathcal{I}$ to be dependent on a set $\mathbf{Y}$ of variables in the SCM $\mathcal{M}$ so that the intentional intervention is conditional on such set: $\mathcal{I}(\mathcal{M} \vert \mathbf{Y})$.}

This desideratum requires that intentional interventions have an explicit reference to a state observed by an agent. This reflects the understanding that agents have awareness of the state of the system through sensors and use this information to choose their course of action. Absence of state-awareness would make interventions independent from the state of the system, thus making the agent unable to respond to and control the system dynamically.

\paragraph{(D3) Goal-directedness.} \emph{It must be explicit that the agent's behavior is aimed at achieving a desired state expressed within the model.}

\rebuttal{Formally, we require the existence of a configuration $\mathbf{x}$ of the SCM, such that the agent maximizes the probability of observing $\mathbf{X}=\mathbf{x}$ in $\mathcal{I}(\mathcal{M} \vert \mathbf{Y})$.}

This desideratum requires that intentional interventions have an explicit reference to the goals of agents. This reflects the understanding that intention is realized in behaviors towards a goal \citep{anscombe1957, dennett1987}. Absence of goal-directedness would make interventions independent from a goal, thus reducing the agent to a random or inconsistent operator.

\paragraph{(D4) Explaining data.} \emph{Models should help us make sense of data which are otherwise unexplained by a given causal model.}

\rebuttal{Formally, data $\mathcal{D}$ generated by the intentionally intervened model should be Markovian wrt $\mathcal{I}(\mathcal{M} \vert \mathbf{Y})$.}

This desideratum encodes the idea that spurious dependencies in the data may arise from a correct causal model in case of an intentional intervention being applied on it. Also, it implies that interventions exist in the world and do produce such dependencies without modifying permanently the causal relations connecting variables. This desideratum is linked to (D1) as it presupposes a real distinction between observational and interventional quantities.

\paragraph{(D5) Discriminating intentions.} \emph{Models should help us isolate the intended effects of an agent among all the effects obtained by an intervention.}

\rebuttal{Formally, given two goals $\mathbf{X}=\mathbf{x}$ and $\mathbf{X'}=\mathbf{x'}$, data $\mathcal{D}$ from the intentionally intervened model $\mathcal{I}(\mathcal{M} \vert \mathbf{Y})$ should help us identify the goal of the agent.}

This desideratum encodes one of the main aims of reasoning about agents, that is understanding their goals starting from observable behavior. It implies that, even when the agent knows about all the effects produced by its behavior, it only aims at a subset of the produced effect \citep{compagno2025}. Identifying this subset is a crucial task of teleological inference.

\paragraph{(D6) Acyclicity.} \emph{Models should agree with the acyclicity assumption of standard SCMs.}

\rebuttal{Formally, an intentionally intervened model $\mathcal{I}(\mathcal{M} \vert \mathbf{Y})$ admits a DAG.}

This desideratum ensures that teleological inference remains compatible with causal inference in its most standard form (without causal loops). Also, it imposes that teleological reasoning does not contradict the current understanding of causation, in which actual causes precede their effects \citep{reichenbach1956, faye2025}.

\paragraph{(D7) Time-agnostic.} \emph{Models should allow us to work with data with no explicit time dimension.}

\rebuttal{Formally, data $\mathcal{D}$ generated by an intentionally intervened model $\mathcal{I}(\mathcal{M} \vert \mathbf{Y})$ is i.i.d..}

This desideratum implies that teleological inference is not strictly about the temporal evolution of a modeled system, and does not require repeated measurements. This means that the goals of an action can be inferred instantaneously, whenever disposing of relevant data.

\subsection{Assumptions for intention modelling}

In this section we further discuss the assumptions we introduced in \Cref{ssec:assumptions}.

\setcounter{assumption}{0}
\begin{assumption}[Perfect knowledge]
    The intentional agent knows the true causal structure (DAG) of the causal system it is interacting with.
\end{assumption}

What grounds teleological inference is that actual interventions on causal systems must produce data according to causal structure. That is, interventions and their goals leave a trace in data collection, and this happens under any condition of intervention (be it finalized, inconsistent, arbitrary or even totally random). That said, the easiest way to interpret data as the result of interventions is to assume that the outcome is completely transparent to the agent, and so that the entirety of its behavior is effective for reaching its goals, without any effort lost because of mismaneuvering.

\begin{assumption}[High-frequency agent]
    The intentional agent acts at a significantly higher frequency than the exogenous variables.
\end{assumption}

This assumption makes sense intuitively: agents could not exert control if the world \dd{slipped out of their hands} too quickly. From a formal perspective, it implies that exogenous variables assign the same values to the factual and to the counterfactual twins of the SFM.

\section{Intentional Interventions and Teleological Inference}\label{app:definition}

In this section we discuss the definition of intentional intervention and the results of using it for teleological inference (agent detection and intention discovery).

\subsection{Intentional Interventions}

\setcounter{mydef}{5}
\begin{mydef}[Intentional intervention]
    {Given an SCM $\scm$, an intervention variable $X\in \envars$, a variable $Y \in \envars$ such that $Y \in \desc(X)$, and a function $f(y)$, an \emph{intentional intervention} $\iiintv{X^\star}{{f}(y)}$ is an operator applied to $\scm$ and generating a new \emph{final SCM} $\finscm$ constructed via a two-step procedure:}

    \begin{enumerate}
        \item {Replica: infer the distribution ${\probdistsu{}}_{\vert Y=y}$ in $\scm$ under the conditioning $Y=y$ and define a new model $\scmu{\rep} = \myset{\envars \cup \envars^\star, \exvars, \structfuncs\cup\structfuncs^\star, {\probdistsu{}}_{\vert Y=y}}$ where $\envars^\star, \structfuncs^\star$ replicates $\envars,\structfuncs$.}

        \item {Intervention: perform intervention $\dointv{X^\star}{f(y)}$ on $\scmu{\rep}$, generating SCM $\finscm$.}
    \end{enumerate}
\end{mydef}

We suggest that intentionally intervened variables do not listen to \emph{any} of their causes. At a first sight, this may sound implausible: one could argue that, before turning on the heater, one looks outside of the window (and so some variable such as the weather influences action). Carefully looking, this is false. In fact, one may turn on the heater whenever he or she expects to feel cold, without looking at the weather, and possibly \emph{despite} of the fact that the sun is shining outside.

More precisely, action only depends on the goal, and any other variable is just a way to assess if action is needed, that is, to estimate the counterfactual course of the world in case no action is taken. In our approach, agents \dd{listen to} what would happen to the aimed variable, and to do so they make an inference based on all the causal knowledge included in the causal model. In other words, by \dd{listening to} the counterfactual temperature, the agent is implicitly evaluating the impact of all of its causes in the model, notably weather. But these causes only play a role in counterfactual evaluation, and not directly on intervention, which is solely \dd{pulled} by its aim: if the agent does not want to feel cold, and estimates that action is needed, then he or she will turn on the heater \emph{no matter what} is the actual state of any cause. As an example, an agent making a mistake in counterfactual evaluation would turn on the heater even if there was actually no need for it (but this situation is excluded from the scope of the present paper by Assumption \ref{ass:perfect-knowledge}).

\paragraph{Irrational and immoral actions} About intentional interventions, we wrote: \emph{Importantly, notice that we rely on a deterministic function f, but we do not require it to be optimal or the agent to be moral or rational (\emph{contra} Halpern and Kleiman-Weiner, 2018), allowing us to capture the behavior of agents with limited computational capacity, affected by biases, or who are different from us.}

Most work on agency tries to understand it starting from general assumptions of rationality or morality. For example, by trying to state that turning on the heater is the most rational thing to do, in case we do not want to feel cold and the weather is bad. This \emph{top-down} approach has several known limitations, as human (and other) agents almost never act rationally \citep{KahnemanTversky1974}. Such modeling attempts are therefore faced with the constant risk of missing the reality of empirical practices.

In our \emph{bottom-up} approach, instead, any kind of action can be detected, and any kind of intention can be discovered. For example, we may find empirically that somebody smokes in order to get lung damage, and this despite any of our considerations about the agent's rationality or morality.
As long as the agent and the investigator share the same correct causal understanding of the world (Assumption \ref{ass:perfect-knowledge}), actions and their aims are transparent. Future work might deal with model misspecification, taking into account that either the agent's or the investigator's causal knowledge may be wrong, and so in some cases possibly preventing the exact discovery of intentions from data alone.

\subsection{Teleological Inference}

\setcounter{myprop}{1}
\begin{myprop}[Identifiability for detecting agents in simple colliders]
    Given an instance of the problem $\langle\scmdag,\mathcal{D}\rangle$ of detecting an intentional agent in a simple collider, if the agent intervenes on one of the parents $\parents(C)$ of the collider $C$, an intentional intervention can be detected from $\mathcal{D}$ as a violation of Markovianity in $\scmdag$.
\end{myprop}

We claim that intentional interventions can be detected whenever they assign values to simple colliders, because this inevitably produces violations of Markovianity, which can be measured. There are some limit cases:
\begin{itemize}

\item Imagine an agent assigning values completely at random, possibly with the aim of masking its presence. In this case, Markovianity would not be violated, but we could observe a distribution shift in the model. 

\item The agent would have better chances at passing unnoticed, if it assigned values exactly in the way the causal system would do without any intervention. Conceptually, this would still be an intervention (there would be an agent assigning values), but it would leave no trace in the data and so it would be impossible to detect empirically. This problem has been treated philosophically by L. Wittgenstein, G.E.M. Anscombe and others as the aporia of distinguishing between the event of an anesthetized arm falling down and the action of somebody who makes his or her own arm descend exactly at the same speed \citep{hornsby1980}.

\item A similar issue pertains to perfect interventions. If the constant value assigned to the intervened variable corresponds to the unintervened value that the causal system would assign on its own (that is, the variable only takes one value even in case of no intervention), then the perfect intervention falls into the class of the undetectable interventions.
\end{itemize}

\begin{myprop}[Identifiability for intention discovery in simple forks]
    Given an instance of the problem $\langle \scm,\finscm \rangle$ of discovering an intention in a simple fork, the argument $y$ of the intentional intervention $\iiintv{X^\star}{{f}(y)}$ can be identified by intervening on the descendants of $X$.
\end{myprop}

Intentions may be discoverable either from interventional (experimental) or observational data. In the present paper we only show how to identify intentions with interventional data (Appendixes \ref{app:proofs} and \ref{app:simulations}) and leave observational data to future work.

\section{Proofs}\label{app:proofs}

In this section we provide proofs for our propositions and lemmas.

\setcounter{myprop}{0}
\begin{myprop}
    SFMs implies DAGs.
\end{myprop}

\emph{Proof.}
Let $\scm=\scmsignature$ be the original SCM, and let $\finscm=\finscmsignature$ be the final SCM produced by the intentional intervention $\iiintv{X^\star}{f(y)}$; let also $\scmdag$ be the DAG over $\envars$, and $\mathcal{G}_{\mathcal{M}^\star}$ the DAG over $\envars^{\star}$.

Let us recall the definition of a directed path in a directed graph $\langle \mathcal{V}, \mathcal{E} \rangle$ defined over vertices $\mathcal{V}$ and edges $\mathcal{E}$. A directed path $p$ is an ordered set of vertices $(X_1,...,X_n)$, for $n \geq 1$, $X_i \in \mathcal{V}$ and $(X_i,X_{i+1}) \in \mathcal{E}$. A directed path $p=(X_1,...,X_n)$ is a cycle if $X_1=X_n$. A directed graph is acyclic (DAG) if it does not admit any cycle.

To prove that a SFM implies a DAG, we need to show that $\scmdagu{\fin}$ admits no cycle. Let us consider a generic directed path $p=(X_1,...,X_n)$; for a cycle to exist it must hold that $X_1=X_n$, with $X_1,X_n$ elements of $\envars$ or $\envars^{\star}$. Let us consider the two possibilities:
\begin{enumerate}
    \item $X_1, X_n \in \envars$: a cycle from an endogenous variable in $\scmdag$ to itself can either be confined to $\scmdag$ or span $\scmdag$ and $\mathcal{G}_{\mathcal{M}^\star}$. However, if it is confined to $\scmdag$ then no cycle is admissible because $\scmdag$ is by definition a DAG; if the path crosses from $\scmdag$ to $\mathcal{G}_{\mathcal{M}^\star}$ then by construction no edge is present to return to $\scmdag$.
    \item $X_1, X_n \in \envarsu{\fin}$: a cycle from an endogenous variable in $\mathcal{G}_{\mathcal{M}^\star}$ to itself must be confined to $\mathcal{G}_{\mathcal{M}^\star}$, as no edge is present to transition to $\scmdag$. Thus, if confined to $\mathcal{G}_{\mathcal{M}^\star}$, no cycle is admissible because $\mathcal{G}_{\mathcal{M}^\star}$ is by definition a DAG.
\end{enumerate}

Therefore, in both cases, no cycle can be defined. Hence $\scmu{\fin}$ is a DAG. $\blacksquare$


\setcounter{mylemma}{0}
\begin{mylemma}
    A counterfactual $\ctf{X}{x}{X}{x^*}$ is an intentional intervention $\iiintv{X^\star}{f(y)}$ for which: (i) \Cref{ass:high-frequency} holds perfectly; (ii) $Y=X$; (iii) $f(y)=x^*$.
\end{mylemma}

\emph{Proof.} In order to show the identity of the two operators, $\ctf{X}{x}{X}{x^*}$ and $\iiintv{X^\star}{f(y)}$, under conditions (i-iii), it will be sufficient to show that the construction of the respective models, $\ctfscm, \finscm$, as given by \Cref{def:twin_model} and \Cref{def:intentional_intervention}, are identical.

First of all, notice that condition (i) is necessary for the computation of counterfactual.

Next, let us consider step 1 in the construction procedure of \Cref{def:twin_model} and \Cref{def:intentional_intervention}. The two constructions are identical as soon as the conditioning sets are identical, which is guaranteed by condition (ii). Both construction performs an identical abduction since $Y=X$. Also notice that, by definition, $X$ is a descendant of itself, $X\in\desc(X)$, so \Cref{def:intentional_intervention} is satisfied.

Last, step 2 of the construction procedure is forced to be identical by condition (iii), which requires the intentional intervention to be a constant $\iiintv{X^\star}{x^*}$.

Thus, the two procedures construct identical models $\ctfscm = \finscm$. $\blacksquare$\\

Notice that, with respect to the discussion in \Cref{ssec:intentional_interventions}, this identity hold as long as the experimenter using an intentional intervention interprets the variables $\envars$ as the actual observed world, and the model $\envars^\star$ at the counterfactual imaginary world. 

\begin{mylemma}
    A perfect intervention $\dointv{X}{x}$ is an $\iiintv{X^\star}{f(y)}$ intentional intervention for which: (i) \Cref{ass:high-frequency} does not necessarily hold; (ii) $Y=\emptyset$; (iii) $f(y)=x$. 
\end{mylemma}

\emph{Proof.} A perfect intervention $\dointv{X}{x}$ can be written in counterfactual notation as $\ctfinv{X}{x}$, that is, a case in which the counterfactual world has no constraints from the actual world; the counterfactual world will then simply represent the intervention of interest.

As for counterfactual, we will show that the identity of the two operators, $\ctfinv{X}{x}$ and $\iiintv{X^\star}{f(y)}$, under conditions (i-iii), is entailed by the construction of the respective models, $\ctfscm, \finscm$, according to \Cref{def:twin_model} and \Cref{def:intentional_intervention}.

Differently from proper counterfactuals, notice that condition (i) is now not necessary, since we are computing interventions on a single world.

Let us move to step 1 in the construction procedure of \Cref{def:twin_model} and \Cref{def:intentional_intervention}. The two constructions are identical as soon as the conditioning sets are identical, which is guaranteed by condition (ii), $Y=\emptyset$. In particular, neither is conditioning on anything, thus leaving the probability distributions of the original model unchanged.

Last, step 2 of the construction procedure is forced to be identical by condition (iii), which requires the intentional intervention to be a constant $\iiintv{X^\star}{x}$.

Thus, the two procedures construct identical models $\ctfscm = \finscm$. $\blacksquare$\\

We can also formalize a relation to dynamic treatments as follows:
\setcounter{theorem}{3}
\begin{mylemma} \label{lem:lemma3}
    A dynamic treatment $\dointv{X}{{g}({z})}$ is an intentional intervention $\iiintv{X^\star}{f(y)}$ for which: (i) the dynamic SCM on which the dynamic treatment is applied does not have edges across timesteps; (ii) the time-horizon of the dynamic SCM is $T=2$; (iii) \Cref{ass:high-frequency} does not necessarily hold; (iv) ${Y}={Z}$; (v) $f(y)=g(z)$.
\end{mylemma}

\emph{Proof.} Aligning dynamic treatments and intentional interventions is less straightforward than the alignment with perfect interventions or counterfactuals. This is due to the problem that dynamic treatments and intentional interventions are applied to different objects, that is, to dynamic SCM and SCM, respectively.
So, first of all we need to align these two models. Conditions (i)-(ii) are used to reduce the dynamic SCM to which we apply $\dointv{X}{{g}({z})}$ to a trivial SCM by removing inter-step arrows and limiting the time-steps to 2. This allows us to align the first and the second timestep of the dynamic SCM to the twin construction of an intentional intervention.

Further, as in the case of interventions, the identity of the values of the exogenous variables is not guaranteed in a dynamic treatment; this is captured by condition (iii).

At this point, all that is left to show is that dynamic treatment and intentional intervention generate the same models. Under condition (iv)-(v) it is immediate to show that the same model is generated, as in the proof for \Cref{lem:counterfactual} and \Cref{lem:intervention}. $\blacksquare$\\

\setcounter{theorem}{3}
\begin{myprop}[Identifiability for detecting agents in simple colliders]
    Given an instance of the problem $\langle\scmdag,\mathcal{D}\rangle$ of detecting an agent in a simple collider, if the agent intervenes on one of the parents $\parents(C)$ of the collider $C$, an intentional intervention can be detected from $\mathcal{D}$ as a violation of Markovianity in $\scmdag$.
\end{myprop}

\emph{Proof.} Consider the generic model in \Cref{fig:proof1} representing a simple collider affected by the intentional intervention $\iiintv{X^\star}{f(y)}$, which corresponds to the agent having knowledge of $Y$ and intervening on $X\in\parents(Y)$. In the intentional intervention setup we have knowledge of the causal graph structure $\scmdag$ and a dataset $\mathcal{D}=\setNelems{X^\star_i,Y^\star_i,Z^\star}$ with a violation of Markovianity wrt $\scmdag$ among the parents $\parents(Y)$ of $Y$.

In the original model $\scmdag$, Markovianity requires that the independencies in the graphical model are preserved in the data; specifically, it must hold that, since the parent $X$ is d-separated from the other parent $Z$, $X \perp_\mathcal{G} Z$, then $X$ and $Z$ must be independent in the data, $X \perp_\mathcal{D} Z$. However, the dataset $\mathcal{D}$ contains a violation of Markovianity, meaning that $X$ is not independent from $Z$.

Markovianity is instead explained and restored in the SFM, as the variables corresponding to the observed data, $X^\star$ and $Z^\star$, are not d-separated anymore; indeed, it is easy to identify a backdoor path between them: $X^\star \leftarrow Y \leftarrow Z \leftarrow U_Z \rightarrow Z^\star$.

Notice that the proof trivially generalizes to simple colliders with more parents, as a violation will still take place between the intentionally intervened parent and the other parents. $\blacksquare$\\

\begin{figure}
    \centering
    \begin{tikzpicture}
        [node distance = 1cm, thick,
        arrow/.style = {draw, ->},
        circ/.style = {draw, ellipse}]

        \begin{scope}[xshift=0]
            \node[circ, text=black, draw=black, fill=white, dashed] at (-1.2, 1) (H) {$X$};
            \node[circ, text=black, draw=black, fill=white, dashed] at (0, 0) (T) {$Y$};
            \node[circ, text=black, draw=black, fill=white, dashed] at (1.2, 1) (W) {$Z$};

            \node[circ, text=black, draw=black, fill=white, dashed] at (1.2, 2.5) (UW) {$U_Z$};
            \node[circ, text=black, draw=black, fill=white, dashed] at (0, -1.5) (UT) {$U_Y$};
            \node[circ, text=black, draw=black, fill=white, dashed] at (-1.2, 2.5) (UH) {$U_X$};
    
            \node[circ, text=black, draw=black, fill=white] at (2.8, 1) (H') {$f$};
            \node[circ, text=black, draw=black, fill=white] at (4, 0) (T') {$Y^\star$};
            \node[circ, text=black, draw=black, fill=white] at (5.2, 1) (W') {$Z^\star$};        
            
            \draw[arrow] (W) edge [bend left=0] (T);
            \draw[arrow] (H) edge [bend left=0] (T);
    
            \draw[arrow] (W') edge [bend left=0] (T');
            \draw[arrow] (H') edge [bend left=0] (T');

            \draw[arrow] (UW) edge [bend left=0] (W);
            \draw[arrow] (UH) edge [bend left=0] (H);
            \draw[arrow] (UT) edge [bend left=0] (T);
            
            \draw[arrow] (UW) edge [bend left=10] (W');
            \draw[arrow] (UT) edge [bend left=-10] (T');
            
            \draw[arrow] (T) edge [bend left=-10] (H');

        \end{scope}

        \end{tikzpicture}
    \caption{SFM for detecting intentional interventions in colliders.}
    \label{fig:proof1}
\end{figure}
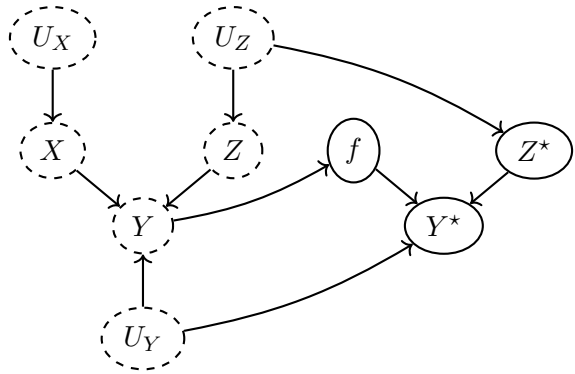

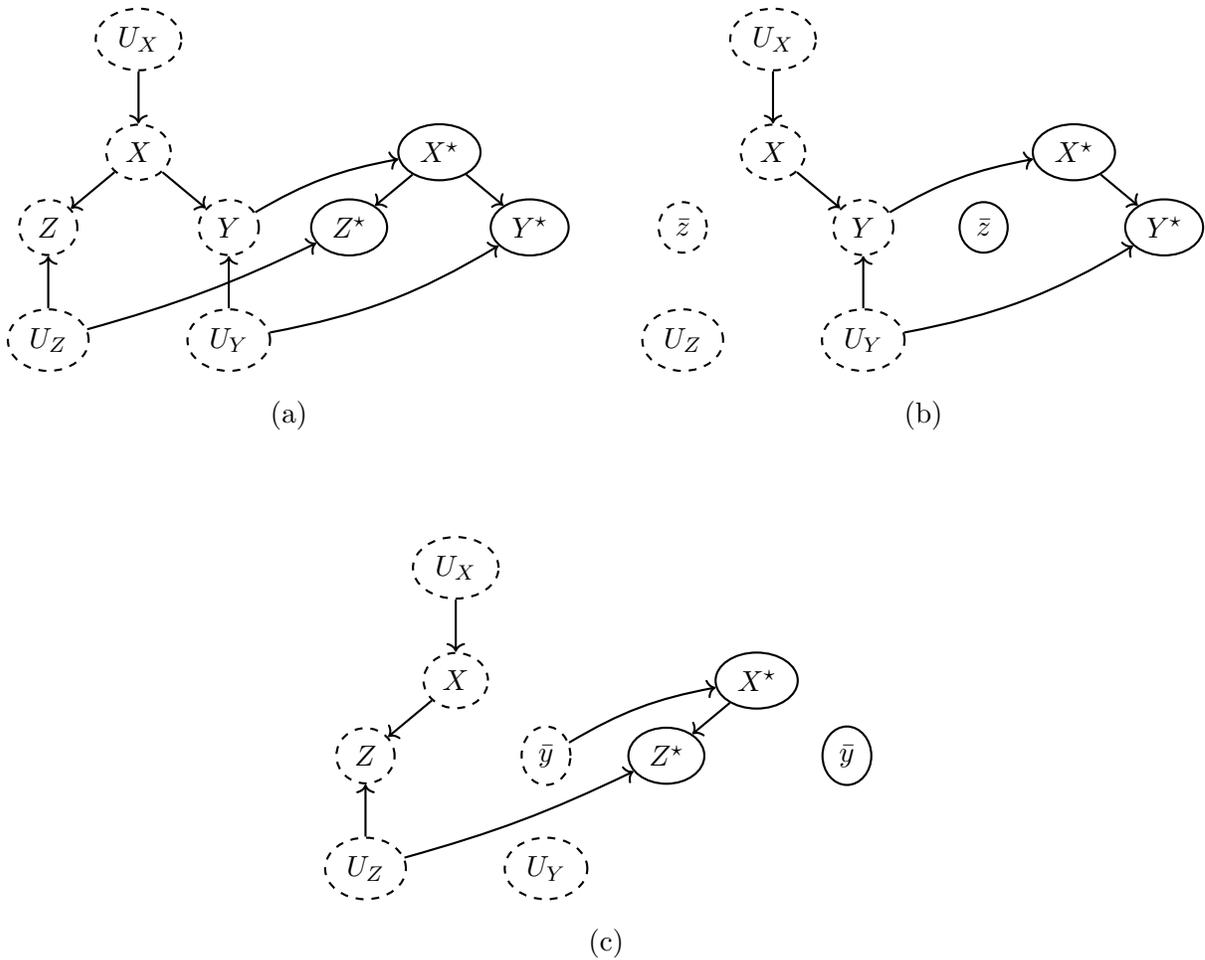
\begin{figure}
    \centering
    \begin{tikzpicture}
        [node distance = 1cm, thick,
        arrow/.style = {draw, ->},
        circ/.style = {draw, ellipse}]

        \begin{scope}[xshift=0]
            \node[circ, text=black, draw=black, fill=white, dashed] at (-1.2, -1) (D) {$Z$};
            \node[circ, text=black, draw=black, fill=white, dashed] at (0, 0) (Y) {$X$};
            \node[circ, text=black, draw=black, fill=white, dashed] at (1.2, -1) (P) {$Y$};
            
            \node[circ, text=black, draw=black, fill=white, dashed] at (-1.2, -2.5) (UD) {$U_Z$};
            \node[circ, text=black, draw=black, fill=white, dashed] at (0, 1.5) (UY) {$U_X$};
            \node[circ, text=black, draw=black, fill=white, dashed] at (1.2, -2.5) (UP) {$U_Y$};
    
            \node[circ, text=black, draw=black, fill=white] at (2.8, -1) (D') {$Z^\star$};
            \node[circ, text=black, draw=black, fill=white] at (4, 0) (Y') {$X^\star$};
            \node[circ, text=black, draw=black, fill=white] at (5.2, -1) (P') {$Y^\star$};        
            
            \draw[arrow] (Y) edge [bend left=0] (P);
            \draw[arrow] (Y) edge [bend left=0] (D);
    
            \draw[arrow] (Y') edge [bend left=0] (P');
            \draw[arrow] (Y') edge [bend left=0] (D');
            
            \draw[arrow] (UY) edge [bend left=0] (Y);
            \draw[arrow] (UP) edge [bend left=0] (P);
            \draw[arrow] (UD) edge [bend left=0] (D);
            
            \draw[arrow] (UP) edge [bend left=-10] (P');
            \draw[arrow] (UD) edge [bend left=-5] (D');
            
            \draw[arrow] (P) edge [bend left=10] (Y');

            \node at (2,-3.5) {(a)};
        \end{scope}

        \begin{scope}[xshift=240]
            \node[circ, text=black, draw=black, fill=white, dashed] at (-1.2, -1) (D) {$\bar{z}$};
            \node[circ, text=black, draw=black, fill=white, dashed] at (0, 0) (Y) {$X$};
            \node[circ, text=black, draw=black, fill=white, dashed] at (1.2, -1) (P) {$Y$};
            
            \node[circ, text=black, draw=black, fill=white, dashed] at (-1.2, -2.5) (UD) {$U_Z$};
            \node[circ, text=black, draw=black, fill=white, dashed] at (0, 1.5) (UY) {$U_X$};
            \node[circ, text=black, draw=black, fill=white, dashed] at (1.2, -2.5) (UP) {$U_Y$};
    
            \node[circ, text=black, draw=black, fill=white] at (2.8, -1) (D') {$\bar{z}$};
            \node[circ, text=black, draw=black, fill=white] at (4, 0) (Y') {$X^\star$};
            \node[circ, text=black, draw=black, fill=white] at (5.2, -1) (P') {$Y^\star$};        
            
            \draw[arrow] (Y) edge [bend left=0] (P);
    
            \draw[arrow] (Y') edge [bend left=0] (P');
            
            \draw[arrow] (UY) edge [bend left=0] (Y);
            \draw[arrow] (UP) edge [bend left=0] (P);
            
            \draw[arrow] (UP) edge [bend left=-10] (P');
            
            \draw[arrow] (P) edge [bend left=10] (Y');

            \node at (2,-3.5) {(b)};
        \end{scope}

        \begin{scope}[xshift=120, yshift=-200]
            \node[circ, text=black, draw=black, fill=white, dashed] at (-1.2, -1) (D) {$Z$};
            \node[circ, text=black, draw=black, fill=white, dashed] at (0, 0) (Y) {$X$};
            \node[circ, text=black, draw=black, fill=white, dashed] at (1.2, -1) (P) {$\bar{y}$};
            
            \node[circ, text=black, draw=black, fill=white, dashed] at (-1.2, -2.5) (UD) {$U_Z$};
            \node[circ, text=black, draw=black, fill=white, dashed] at (0, 1.5) (UY) {$U_X$};
            \node[circ, text=black, draw=black, fill=white, dashed] at (1.2, -2.5) (UP) {$U_Y$};
    
            \node[circ, text=black, draw=black, fill=white] at (2.8, -1) (D') {$Z^\star$};
            \node[circ, text=black, draw=black, fill=white] at (4, 0) (Y') {$X^\star$};
            \node[circ, text=black, draw=black, fill=white] at (5.2, -1) (P') {$\bar{y}$};        
            
            \draw[arrow] (Y) edge [bend left=0] (D);
    
            \draw[arrow] (Y') edge [bend left=0] (D');
            
            \draw[arrow] (UY) edge [bend left=0] (Y);
            \draw[arrow] (UD) edge [bend left=0] (D);
            
            \draw[arrow] (UD) edge [bend left=-5] (D');
            
            \draw[arrow] (P) edge [bend left=10] (Y');

            \node at (2,-3.5) {(c)};
        \end{scope}

        \end{tikzpicture}
    \caption{(a) SFM for intention discovery in forks. (b) SFM with intervention $\dointv{Z}{\bar{z}}$ for intention discovery in forks. (c) SFM with intervention $\dointv{Y}{\bar{y}}$ for intention discovery in forks.}
    \label{fig:proof2}
\end{figure}

\begin{myprop}[Identifiability for intention discovery in simple forks]
    Given an instance of the problem $\langle \scm,\finscm \rangle$ of intention discovery in a simple fork, the argument $y$ of the intentional intervention $\iiintv{X^\star}{{f}(y)}$ can be identified by intervening on the descendants of $X$.
\end{myprop}

\emph{Proof.} Consider the generic model in \Cref{fig:proof2}(a) representing a simple fork affected by the intentional intervention $\iiintv{X^\star}{f(y)}$, which corresponds to the agent having knowledge of $Y$ and intervening on $X$. In the intention discovery setup we have knowledge of the causal graph structure $\scmdag$, knowledge of an intentional intervention happening on $X^\star$, and a dataset $\mathcal{D}=\setNelems{X^\star_i,Y^\star_i,Z^\star}$.

Let us now consider intervening on the descendant $\desc(X)$ of $X$.
First, consider intervening on $Z$ by $\dointv{Z}{\bar{z}}$ and then letting the agent take its intentional intervention. This leads to the SFM in \Cref{fig:proof2}(b). It is immediate to see that we can now empirically evaluate whether $Z$ affects $X^\star$ via causal discovery. By definition of (interventional) causality, $\nexists \bar{z}, z'$, $\bar{z} \neq z'$ such that, ceteribus paribus, $P(X^\star,Y^\star,\bar{z}) \neq P(X^\star,Y^\star,z')$.

Conversely, let us now consider intervening on $Y$ by $\dointv{Y}{\bar{y}}$ and then letting the agent take its intentional intervention. This leads to the SFM in \Cref{fig:proof2}(c). It is immediate to see that we can now empirically evaluate whether $Y$ affects $X^\star$ via causal discovery. By definition of (interventional) causality, $\exists \bar{y}, y'$, $\bar{y} \neq y'$ such that, ceteribus paribus, $P(X^\star,Z^\star,\bar{y}) \neq P(X^\star,Z^\star,y')$.

Notice that the proof trivially generalizes to simple forks with more children, as we can empirically discover which variable affects the intentional intervention of the agent via atomic interventions. $\blacksquare$\\

\section{Simulations}\label{app:simulations}
In this section, we provide illustrative numerical results for the models presented in the examples throughout the main paper. All simulations are openly available at \url{https://anonymous.4open.science/r/FinalModels-C026}.

\subsection{Heating model for intentional agent discovery}
We first implemented the model of \Cref{ex:heating_causal}. 

\paragraph{Causal Model.}
Let $\scmu{caus}=\scmsignature$ be the heating model defined as:
\begin{itemize}
    \item $\envars = \myset{W,T,H}$;
    \item $\exvars = \myset{U_W,U_T,U_H}$;
    \item $\structfuncs = \myset{f_W,f_T,f_H}$ where:
    \begin{itemize}
        \item $f_W = U_W$;
        \item $f_H = U_H$;
        \item $f_T = W+H+U_T$;
    \end{itemize}
    \item $\probdists = \myset{P_W,P_T,P_H}$ where:
    \begin{itemize}
        \item $P_W = \mathtt{Bern}(0.5)$;
        \item $P_H = \mathtt{Bern}(0.5)$;
        \item $P_T = \mathtt{N}(0,10^{-20})$;
    \end{itemize}
    with $\mathtt{Bern}(p)$ being a Bernoulli distribution with probability $p$ and $\mathtt{N}(\mu,\sigma^2)$ being a Normal distribution with mean $\mu$ and variance $\sigma^2$.
\end{itemize}
Notice that the variable temperature ($T$) behaves almost as a perfect deterministic function of heating ($H$) and weather ($W$); however, we add a negligible source of noise to allow for statistical testing.

\paragraph{DAG and d-separations (causal).}
Our causal model immediately implies the DAG $\mathcal{G}_{caus}$ in \Cref{fig:basic-heating-system}(a). We use the DAG to confirm the following elementary d-separations:
\begin{align*}
H\perp_{\mathcal{G}_{caus}} & W\\
H\not\perp_{\mathcal{G}_{caus}} & W\vert T\\
H\not\perp_{\mathcal{G}_{caus}} & T
\end{align*}

\paragraph{Data and independencies (causal).}
We also use the model $\scmu{caus}$ to generate causal data $\mathcal{D}_{caus}$ by sampling. We sample $10^4$ datapoints and then we run independence testing using Fisher z-test and taking a decision at p-value $0.05$ (see the simulations for the actual computed p-values). We confirm the following independencies:
\begin{align*}
H\perp_{\mathcal{D}_{caus}} & W\\
H\not\perp_{\mathcal{D}_{caus}} & W\vert T\\
H\not\perp_{\mathcal{D}_{caus}} & T
\end{align*}

\paragraph{Markovianity.}
The d-separations in the causal graph are reflected by independencies in the causal data. Markovianity is not violated.

\paragraph{Intentional agent.}
We now simulate an intentional agent that decides to turn on the heater in case the temperature is low. We express it in the following mechanism:
\begin{itemize}
    \item $f_{H}^{\star}=\begin{cases}
        1 & \textrm{if }T<0.5\\
        0 & \textrm{else}
\end{cases}$
\end{itemize}

\paragraph{Data and independencies (final).}
We now use the modified program to generate final data $\mathcal{D}_{fin}$ by sampling. Again, we sample $10^4$ datapoints and then we run independence testing using Fisher z-test and taking a decision at p-value $0.05$ (see the simulations for the actual computed p-values). We immediately observe the following:
\begin{align*}
H\not\perp_{\mathcal{D}_{fin}} & W
\end{align*}

\paragraph{Markovianity.}
We face now a discrepancy between the d-separations in the causal graph and the indepenencies from the final data. Markovianity is violated, suggesting that the causal DAG does not explain this data anymore. We use this violation as a pointer to the presence of a intentional intervention.

\paragraph{Final Model.}
We propose to solve the discrepancy by relying on a SFM $\scmu{fin}=\scmsignature$ defined as below:
\begin{itemize}
    \item $\envars = \myset{W,T,H,W^\star,T^\star,H^\star}$;
    \item $\exvars = \myset{U_W,U_T,U_H}$;
    \item $\structfuncs = \myset{f_W,f_T,f_H,f_{W^\star},f_{T^\star},f_{H^\star}}$ where:
    \begin{itemize}
        \item $f_W = f_{W^\star} = U_W$;
        \item $f_H = U_H$;
        \item $f_{H}^{\star}=\begin{cases}
            1 & \textrm{if }T<0.5\\
            0 & \textrm{else}
            \end{cases}$
        \item $f_T = f_{T^\star} = W+H+U_T$;
    \end{itemize}
    \item $\probdists = \myset{P_W,P_T,P_H}$ where:
    \begin{itemize}
        \item $P_W = \mathtt{Bern}(0.5)$;
        \item $P_H = \mathtt{Bern}(0.5)$;
        \item $P_T = \mathtt{N}(0,10^{-20})$.
    \end{itemize}
\end{itemize}

\paragraph{DAG and d-separations (final).}
Our SFM immediately implies the DAG $\mathcal{G}_{fin}$ in \Cref{fig:final-systems}(a). We use the DAG to confirm the following elementary d-separations:
\begin{align*}
H^\star\perp_{\mathcal{G}_{fin}} & W^\star\\
H^\star\not\perp_{\mathcal{G}_{fin}} & W^\star\vert T^\star\\
H^\star\not\perp_{\mathcal{G}_{fin}} & T^\star
\end{align*}
Notice that the d-separations are computed over the \dd{star} variables, as these are the one effectively observed by an experimenter.

\paragraph{Markovianity.}
Using a SFM allows us to restore Markovianity: the d-separations in the final graph are now reflected by independencies in the final data. Thus, teleological reasoning allowed us to detect an agent, and a SFM allows us to make sense of data and the underlying causal model.

\subsection{Smoking model for intention discovery}
Next, we implemented the model of \Cref{ex:smoking_causal}. 

\paragraph{Final Model.}
We define a SFM $\scmu{fin}=\scmsignature$ representing an agent smoking in order to experience pleasure:
\begin{itemize}
    \item $\envars = \myset{S,D,P,S^\star,D^\star,P^\star}$;
    \item $\exvars = \myset{U_S,U_D,U_P}$;
    \item $\structfuncs = \myset{f_S,f_D,f_P,f_{S^\star},f_{D^\star},f_{P^\star}}$ where:
    \begin{itemize}
        \item $f_S = U_S$;
        \item $f_{S}^{\star}=\begin{cases}
            1 & \textrm{if }P>1\\
            0 & \textrm{else}
            \end{cases}$
        \item $f_D = f_{D^\star} = 0.3S + U_D$;
        \item $f_P = f_{T^\star} = 0.5S + U_P + 1$;
    \end{itemize}
    \item $\probdists = \myset{P_S,P_D,P_P}$ where:
    \begin{itemize}
        \item $P_S = \mathtt{Bern}(0.5)$;
        \item $P_D = \mathtt{N}(0,1)$;
        \item $P_P = \mathtt{N}(0,1)$.
    \end{itemize}
\end{itemize}
This SFM implies the DAG $\mathcal{G}_{fin}$ in \Cref{fig:final-systems}(b). 

\paragraph{Final data.}
We now use the SFM $\scmu{fin}$ to generate final data $\mathcal{D}_{fin}$. We will use this data, together with the DAG $\mathcal{G}_{fin}$ and interventional data, to discriminate the intention of the agent acting on the causal system.

First of all, we use the final data $\mathcal{D}_{fin}$ to estimate the probability of smoking $P_{\scmu{fin}}(S^\star)$ in the SFM. \Cref{fig:ID-distributions}(a) shows the empirical distribution computed from the data $\mathcal{D}_{fin}$.

\paragraph{Intervening on the damage.}
We consider first intervening on damage. We implement the standard intervention $\dointv{D}{0}$, representing the removal of dangerous effects from smoking. We then collect data $\mathcal{D}_{fin\vert\dointv{D}{0}}$ from the model $\scmu{fin}$  under our intervention. We then use this data to estimate the probability of smoking $P_{\scmu{fin}}(S^\star \vert \dointv{D}{0})$. \Cref{fig:ID-distributions}(b) shows the empirical distribution computed from the data $\mathcal{D}_{fin\vert\dointv{D}{0}}$. The absence of any change suggests that the intention of the agent is not dependent on damage $D$.

\paragraph{Intervening on the pleasure.}
Last, we consider intervening on pleasure. We implement the standard intervention $\dointv{P}{0}$, representing the removal of pleasurable effects from smoking. We then collect data $\mathcal{D}_{fin\vert\dointv{P}{0}}$ from the model $\scmu{fin}$  under our intervention. We then use this data to estimate the probability of smoking $P_{\scmu{fin}}(S^\star \vert \dointv{P}{0})$. \Cref{fig:ID-distributions}(c) shows the empirical distribution computed from the data $\mathcal{D}_{fin\vert\dointv{P}{0}}$. We now notice a significant distribution shift, pointing to the fact that manipulating pleasure $P$ has a bearing on the intention of the agent.

\begin{figure*}
    \centering
    \begin{tikzpicture}
        [node distance = 1cm, thick,
        arrow/.style = {draw, ->},
        circ/.style = {draw, ellipse}]

        \begin{scope}[xshift=-40]
            \node[] at (0,0) {\includegraphics[scale=0.35]{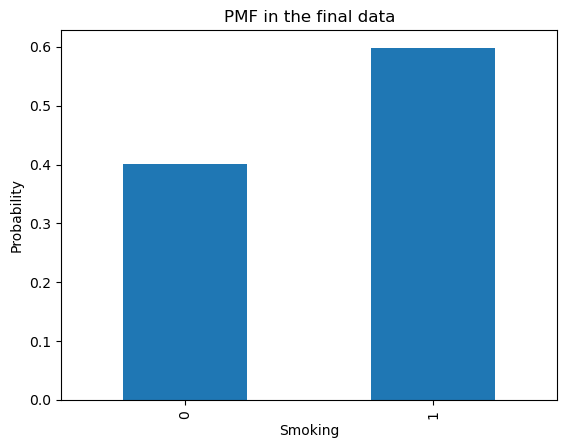}};

            \node at (0,-2.5) {(a)};
        \end{scope}

        \begin{scope}[xshift=120]
            \node[] at (0,0) {\includegraphics[scale=0.35]{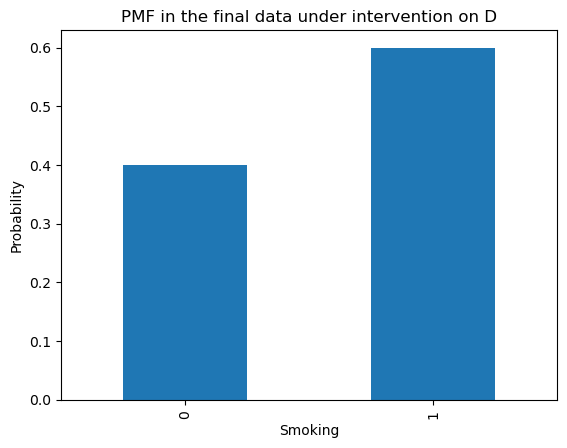}};

            \node at (0,-2.5) {(b)};
        \end{scope}

        \begin{scope}[xshift=280]
            \node[] at (0,0) {\includegraphics[scale=0.35]{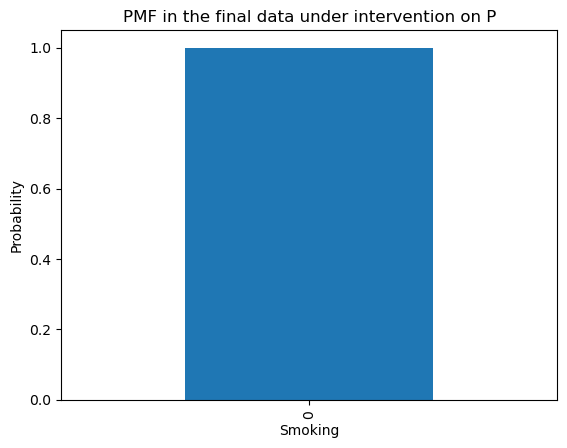}};

            \node at (0,-2.5) {(c)};
        \end{scope}
        
        \end{tikzpicture}
    \caption{Probability mass functions computed during intention discovery: (a) $P_{\scmu{fin}}(S^\star)$; (b) $P_{\scmu{fin}}(S^\star \vert \dointv{D}{0})$; (b) $P_{\scmu{fin}}(S^\star \vert \dointv{P}{0})$.}
    \label{fig:ID-distributions}
\end{figure*}

\end{document}